\documentclass[11pt]{amsart}
\usepackage{natbib}
\usepackage{placeins}
\usepackage[utf8]{inputenc} 
\usepackage[T1]{fontenc}    
\usepackage{hyperref}       
\usepackage{url}            
\usepackage{booktabs}       
\usepackage{amsfonts}       
\usepackage{amsthm}
\usepackage{amsmath}
\usepackage{amssymb}
\usepackage{nicefrac}       
\usepackage{multirow}
\usepackage{microtype}      
\usepackage{xcolor}         
\usepackage{algorithm}
\usepackage{algpseudocode}

\usepackage{graphicx}
\usepackage{subcaption}
\usepackage[margin=1in]{geometry}
\usepackage{float}
\usepackage{bbm}
\usepackage{mathrsfs}
\usepackage{enumerate}
\usepackage{enumitem}
\usepackage{color}
\usepackage{dsfont}
\usepackage[dvipsnames]{xcolor}
\usepackage{hyperref}
\hypersetup{
  colorlinks   = true,
  urlcolor     = blue,
  linkcolor    = blue,
  citecolor   = ForestGreen
}
\setlist[itemize]{leftmargin=25pt, topsep=0pt, itemsep=5pt}
\setlist[enumerate]{label=\textbf{(\arabic*)}, leftmargin=25pt, topsep=0pt, itemsep=5pt}

\title{Detecting Metastable Basins in High Dimensions via Marginal Trajectory Distribution Discrimination}

\author{%
  Taj Jones-McCormick \\
  \texttt{tejonesm@uwaterloo.ca} \\
}

\newtheorem{theorem}{Theorem}
\newtheorem{definition}[theorem]{Definition} 

\begin{document}

\maketitle

\begin{abstract}
    We study the problem of identifying dynamically distinct basins of attraction in high dimensional time-homogeneous Markov processes using only trajectory sampling. This problem is fundamental in the analysis of metastable dynamical systems, where the process rapidly mixes within basins while transitions between basins occur rarely on the timescale of interest, or even when the state space is reducible. Existing approaches typically rely on spatial discretization or spectral analysis of estimated transition operators, which can become unreliable in high dimensional settings or when the underlying basin geometry is highly nonlinear. 
    We propose a discriminative approach to basin identification based on marginal trajectory distribution comparison. We prove a simple risk separation result: if two initial states belong to the same basin, the Bayes-optimal classifier attempting to distinguish their marginal trajectory distributions achieves risk close to 1/2, whereas if they lie in distinct basins, the optimal risk is close to zero. This observation reduces basin detection to a two-sample discrimination problem between marginal trajectory distributions. Motivated by this principle, we develop a neural algorithm that receives a set of candidate basin representatives and iteratively merges them by estimating classification risk with a neural network that approximates the Bayes classifier. We evaluate the method on various metastable systems. These include synthetic systems constructed by embedding low-dimensional dynamics into high dimensional noisy ambient spaces. In these settings, standard spectral and clustering-based methods often fail, while our approach accurately recovers the underlying basin structure. These results highlight a shortcoming of existing methods and highlight trajectory discrimination as an effective and scalable tool for identifying dynamical basins in high dimensional stochastic systems.
\end{abstract}

\section{Introduction}
high dimensional stochastic systems often exhibit multi-modal structure: trajectories concentrate in localized regions of state space and transition only rarely between them. Such behavior arises in a wide range of applications, including Markov chain Monte Carlo methods for sampling multi-modal distributions, Langevin dynamics in rugged energy landscapes, molecular dynamics simulations, and other metastable systems \citep{brooks2011handbook, robert2004monte, Wales_2004, stoltz2010free}. In these settings, understanding the decomposition of the state space into components under which transitions are extremely rare or impossible, is a fundamental problem.

We study this problem in a setting where the state space is high dimensional, continuous and the dynamics are accessible only through trajectory simulation. Let $(X_t)_{t \ge 0}$ be a time-homogeneous Markov process on a continuous state space, $X_t \in \mathbb{R}^d$. We assume only black-box access to the process: given any initialization $x$, we can simulate trajectories forward in time. Our goal is to recover a partition of the state space into components (basins) such that, on a prescribed timescale $\tau$, trajectories starting within a component rarely transition to other components. This formulation also captures reducible processes that decompose into multiple closed communicating classes. 

A large body of existing work approaches this problem by constructing a discretized or low-dimensional representation of the dynamics and analyzing the estimated Markov transition matrix. In finite-state settings, classical techniques identify closed or nearly invariant sets using graph connectivity or spectral decompositions of the transition matrix \citep{norris1998markov, tarjan1972depth, seneta2006non}. In continuous systems, related ideas appear in transfer-operator methods such as diffusion maps \citep{coifman2006diffusion}, Markov state models (MSMs) \citep{prinz2011markov}, and variational approaches including VAMP and VAMPnet \citep{noe2013variational, mardt2018vampnets}. These methods approximate leading eigenfunctions of the Koopman \citep{koopman1931hamiltonian} or Perron–Frobenius operator \citep{farjoun2020laws}. While these approaches have proven powerful in low-dimensional settings, they rely critically on the ability to construct a discretization or representation that preserves the relevant dynamical structure. In high dimensional systems this requirement becomes problematic: direct discretization is infeasible and clustering or embedding based methods may not preserve or respect the geometry of the process's underlying energy landscape. Existing works identify this issue and attempt to bridge this gap \citep{perez2013identification, noe2015kinetic}. However, despite effort to remove reliance on Euclidean distance, these approaches still rely on a brittle discretization of the state space which is not generally sufficient for high dimensional continuous state spaces. Additionally, existing methods are mainly designed for regimes where cross-basin transitions are infrequent but reliably observable in collected trajectory data. In contrast, we do not assume access to data but rather the ability to simulate trajectories and we focus on the regime where cross-basin transitions are extremely rare and may not occur at all.

In this work we propose a different perspective that avoids explicit discretization and global operator estimation. Our starting point is a simple observation: if two initial states lie in the same metastable basin, the distributions of trajectories emanating from those states become nearly indistinguishable on the mixing timescale, whereas trajectories starting in different basins produce distributions that remain easily distinguishable. This observation reduces basin identification to a statistical discrimination problem between marginal trajectory distributions.

Building on this principle, we derive a simple risk separation result showing that the Bayes-optimal classifier distinguishing trajectories from two initializations achieves risk close to $1/2$ when the initializations lie in the same basin and risk close to zero when they lie in distinct basins. This result provides a practical criterion for testing whether two states belong to the same underlying component. Motivated by this characterization, we propose an algorithm that estimates this risk using neural network classifiers. Given a set of candidate basin representatives, we simulate short trajectories from each candidate and train neural classifiers to distinguish between them. Pairs of candidates that cannot be reliably distinguished are merged, producing a refined estimate of the basin partition. 

We evaluate our approach on synthetic high dimensional systems. We demonstrate, by embedding low-dimensional multi-modal dynamics into noisy ambient spaces, that standard spectral and clustering-based approaches fail to recover the correct basin structure in high dimensional, noisy state spaces. In contrast, the proposed discriminative method accurately identifies the underlying components. These experiments demonstrate that trajectory discrimination provides a robust and scalable tool for dynamical mode identification in high dimensional stochastic systems.

\subsection{Contributions}
The main contributions of this work are as follows:
\begin{itemize}

    \item We identify and address a fundamental failure mode of geometry-based basin identification: when dynamics live on low-dimensional manifolds embedded in high dimensional noise, Euclidean structure is misleading
    
    \item We introduce a discriminative framework (and algorithm) for metastable basin identification that reduces the problem to classification between marginal trajectory distributions, avoiding reliance on Euclidean geometry

    \item We empirically demonstrate that this discriminative approach successfully recovers basin structure in high dimensional systems where spectral, variational, and clustering-based baselines struggle.
\end{itemize}

\section{Related Work}

The problem of identifying dynamically coherent or weakly mixing components of a stochastic process has been studied under several names, including metastability, almost-invariant sets, and slow collective variables. Existing approaches broadly fall into two categories: methods based on spectral analysis of transfer operators and methods based on clustering of trajectory data or low-dimensional embeddings. Additionally, these approaches tend to focus on scenarios where transitions between metastable sets are rare but still occurring (in contrast to our set up where transitions may never occur).

In molecular dynamics and statistical physics, metastable states are commonly identified through Markov state models (MSMs) constructed from trajectory data \citep{bowman2013introduction, prinz2011markov}. These methods first discretize the state space using clustering to define microstates and then estimate a transition matrix between them. Spectral analysis of this transition matrix reveals slow dynamical modes, which can be used to identify metastable sets via Perron cluster analysis or related techniques \citep{deuflhard2000identification, sarich2010approximation}. 

More recent work has developed variational approaches to learning slow dynamical modes. The variational approach to conformational dynamics (VAC) and its generalization to nonreversible dynamics (VAMP) formulate the identification of slow modes as an optimization problem over function spaces \citep{noe2013variational, wu2017variational}. Neural network methods such as VAMPnets \citep{mardt2018vampnets} parameterize these functions using deep architectures and optimize a variational objective to approximate dominant singular functions of the Koopman operator. The resulting embeddings are typically interpreted through clustering or soft state assignments to construct coarse-grained dynamical models. Related techniques include time-lagged independent component analysis (TICA) \citep{perez2013identification} and diffusion maps \citep{coifman2006diffusion}. Linear classifiers have also been used to assign states to metastable sets after metastable set candidates are identified \citep{novelli2022characterizing}.

Beyond the molecular dynamics literature, the identification of invariant or almost-invariant sets has also been studied in dynamical systems theory using transfer operator methods and coherent set detection \citep{froyland2013analytic}. In finite-state Markov chains, reducible structure can be identified exactly via graph-theoretic analysis of the transition matrix. However, in high dimensional or continuous state spaces where the dynamics are accessible only through trajectory simulation, the state space is not explicitly enumerable and estimating transfer operators becomes significantly more challenging.

Clustering-based approaches provide an alternative strategy and are frequently used within MSM pipelines. Algorithms such as $k$-means \citep{mcqueen1967some}, spectral clustering \citep{ng2001spectral}, and density-based methods such as DBSCAN \citep{ester1996density} and HDBSCAN \citep{campello2015hierarchical} are often applied either directly to trajectory snapshots or to low-dimensional embeddings obtained from methods such as TICA \citep{perez2013identification}, diffusion maps \citep{coifman2006diffusion}, or variational Koopman embeddings \citep{wu2017variational}. However, clustering methods operate primarily on geometric structure in the data and do not explicitly account for dynamical connectivity. As a result, geometrically separated regions may mix rapidly, while dynamically separated regions may be close in ambient space.

In contrast, our method avoids explicit discretization of the state space and learns embeddings whose interpretation as disjoint dynamical components is indirect and does not require learning global spectral structure of the dynamics. Instead, we adopt a discriminative perspective in which basin identification is reduced to comparing marginal trajectory distributions originating from different initializations. This perspective is closely related to classifier-based two-sample testing and provides a simple criterion for determining whether two states belong to the same dynamical component.

\section{Problem Set Up and Theoretical Motivation}

We consider the problem of identifying dynamically distinct basins of attraction in a Markov process using only trajectory simulations.

Let $(\mathcal{X},\mathcal{B})$ be a measurable state space and let $\{X_t\}_{t\ge0}$ be a time-homogeneous Markov process with transition kernel
\[
P_t(x,A) := \mathbb{P}(X_t \in A \mid X_0 = x),
\quad x\in\mathcal{X},\; A\in\mathcal{B}.
\]

For a measurable set $A\subset\mathcal{X}$, define the first exit time
\[
\tau_A := \inf\{t>0 : X_t \notin A\}.
\]

Our goal is to identify regions of the state space that are dynamically stable on a prescribed time scale. We formalize this notion through the concept of a basin partition.

\begin{definition}[Basin Partition]

Let $P_t$ be a Markov transition kernel.  
A \emph{basin partition} with parameters $(T,\delta)$ consists of

\begin{itemize}
\item a finite collection of measurable sets $W=\{w_1,\dots,w_k\}$ (wells/basins),
\item a corresponding collection of measurable subsets $C=\{c_1,\dots,c_k\}$ (cores),
\end{itemize}

such that:

\begin{enumerate}
\item the wells form a disjoint partition of the state space
\[
\mathcal{X} = \bigcup_{i=1}^K w_i,
\quad w_i \cap w_j = \emptyset \;\; (i\neq j),
\]

\item each core lies strictly inside its well
\[
c_i \subset w_i,
\]

\item trajectories starting from any core rarely leave the corresponding well within time $T$.
\[
\sup_{x\in c_i}
\mathbb{P}(\tau_{w_i} < T \mid X_0=x)
\le \delta \,\,\,\forall \,\,i\in \{1,2,...,k\}.
\]

\end{enumerate}

\end{definition}

This definition captures the metastable structure commonly observed in high dimensional stochastic systems: trajectories initialized within a basin (and away from the boundary, i.e. within a 'core' subset) remain there with high probability for long periods of time.

Given black-box access to the transition kernel $P_t$ and a finite simulation horizon $T$, our goal is to recover such a basin partition.

The previous definition ensures that trajectories rarely leave a basin. We additionally assume that trajectories mix within each basin on a shorter time scale.

\begin{definition}[Conditional Mixing]

Let $w\subset\mathcal{X}$ be a Basin.  
We say that $P_t$ is \emph{$(\varepsilon,t^*)$-conditionally mixing} on $w$ if there exists a probability measure $\pi_w$ supported on $w$ such that

\[
\sup_{x\in c_w}
\left\|
P_t(\cdot \mid \tau_{w} > t, X_0=x)
- \pi_w
\right\|_{TV}
\le \varepsilon
\quad \text{for all } t \ge t^* .
\]

The measure $\pi_w$ is referred to as the quasi-stationary distribution on $w$.

\end{definition}

Where $\|\cdot\|_{TV}$ is the total variation norm. Assuming conditional mixing, trajectories initialized within the same basin become indistinguishable in distribution (on the conditional mixing timescale) when conditioned on remaining in the basin.

\subsection{A Classification Perspective}

The key observation underlying our method is that conditional mixing implies trajectories starting within the same basin quickly become statistically indistinguishable, whereas trajectories initialized in different basins remain separated in distribution due to the rarity of inter-basin transitions.

This suggests a natural statistical test based on classification.

Consider two processes $(X^{(0)}_t)_t$ and $(X^{(1)}_t)_t$ with transition kernel $P_t$, initialized at $X^{(0)}_0=x_0$ and $X^{(1)}_0=x_1$.  
Let $y \sim \mathrm{Bernoulli}(1/2)$ and observe the random pair
\[
(X^{(y)}_{t^*}, y).
\]

A classifier observing $X_{t^*}^{(y)}$ attempts to predict which initial state generated the trajectory.

\begin{definition}[Bayes Classifier and Risk]

The Bayes classifier associated with $(x_0,x_1)$ is defined as
\[
C^{\mathrm{Bayes},t^*}_{(x_0,x_1)}(x)
=
\arg\max_{l\in\{0,1\}}
\mathbb{P}(y=l \mid X^{(y)}_{t^*}=x).
\]

The risk of a classifier $C$ is defined as
\[
R(C)
=
\mathbb{P}\big(C(X^{(y)}_{t^*}) \neq y\big).
\]

\end{definition}

The Bayes classifier minimizes classification risk over all measurable classifiers.

The following theorem shows that classification risk separates states belonging to the same basin from those belonging to different basins.

\begin{theorem}[Basin Identifiability]\label{thm}

Let $(W,C)$ be a basin partition with parameters $(T,\delta)$.  
Assume each well $w\in W$ satisfies $(\varepsilon,t^*)$-conditional mixing for some $t^*<T$ uniform over $w\in W$.

Let $x_0,x_1 \in \bigcup_{w \in W} c_w$ and consider the Bayes classifier $C^{\mathrm{Bayes},t^*}_{(x_0,x_1)}$.

Then the following holds:

\begin{enumerate}

\item If $x_0,x_1 \in c_w$ for some $w \in W$, then
\[
R(C^{\mathrm{Bayes},t^*}_{(x_0,x_1)})
\ge
(1-\delta)\left(\frac{1}{2}-\varepsilon\right).
\]

\item If $x_0 \in c_w$ and $x_1 \in c_{w'}$ with $w \neq w'$, then
\[
R(C^{\mathrm{Bayes},t^*}_{(x_0,x_1)})
\le
\delta.
\]

\end{enumerate}

\end{theorem}

The proof is provided in Appendix~\ref{appendix:thm_proof}.  
The theorem shows that when intra-basin mixing occurs rapidly and inter-basin transitions are rare, the Bayes classification risk exhibits a clear separation depending on whether two initial states lie in the same basin. We emphasize that the purpose of this result is not theoretical novelty but to motivate an algorithm. This perspective connects basin identification to the classical problem of two-sample testing: trajectories generated from different initial states induce distributions over future states, and identifying basins amounts to determining whether these distributions are statistically distinguishable. In this sense, our approach (outlined below) can also be viewed as a form of contrastive learning over marginal trajectory distributions, where the classifier learns features that distinguish dynamical 
modes of the process.

\subsection{From Risk Separation to Basin Identification}

Theorem~\ref{thm} suggests the following conceptual procedure.  
Given two initial states, we simulate trajectories, train a classifier to predict which initial state generated each trajectory, and use the resulting classification risk to determine whether the states belong to the same basin.

Low classification risk indicates that the marginal trajectory distributions are distinguishable, suggesting the states belong to different basins. Conversely, risk close to $1/2$ indicates that the distributions are nearly identical, suggesting the states lie in the same basin.

In practice, the Bayes classifier is unknown. We therefore approximate it using a neural network trained on simulated trajectory data.

\subsection{A Practical Algorithm}

Our proposed method operates in two stages. An initial 'discovery stage' with the intent of producing a high recall (but possibly low precision) estimate of a set of 'candidate basins'. The second stage is 'basin refinement' which takes the set of basin candidates produced by the discovery stage and refines the estimate by merging candidate basins whose marginal trajectory distributions cannot be distinguished by a trained classifier.

\subsubsection{Basin Discovery}

The first stage of our proposed method requires producing a high recall set of candidate basins. One simple approach to this stage is to sample many independent trajectories from random initializations and take the endpoint of each trajectory as a basin candidate. This 'discovery' stage of the algorithm can be substituted with various techniques such as clustering initial trajectories, augmenting the process dynamics with exploration rewards \citep{burda2018exploration} or incorporating prior information on potential basin locations. We experiment with a few of these approaches but ultimately found it most effective to simply use random trajectory endpoints and thus we only report these results. Note that in the general case, identifying all possible basins of an arbitrary density can be an exponentially hard problem in the dimension. Solving general basin discovery is not our goal. We are rather concerned with tractable scenarios where our process gives rise to few candidate basins which are discoverable with high probabilities from some initial distribution.

\subsubsection{Basin Refinement}

Below is the procedure for the second stage of our algorithm, basin refinement. Combining a method of basin discovery with the Neural Basin Refinement algorithm above, we introduce the Neural Basin Identification algorithm (NBI). 

\begin{algorithm}[H]
\caption{Neural Basin Refinement}
\begin{algorithmic}[1]
\Require Markov Transition Kernel $P_t$
\Require Candidate states $x_1,\dots,x_K$
\Require Simulation horizon $t^*$, threshold $\gamma$, Number of simulations per trajectory $n$

\For{each candidate state $x_i$}
\State Simulate and store $n$ iid trajectories $(X_t^{(i)})_{t=0}^{t^*}$ initialized from $x_i$
\EndFor

\For{each pair $(x_i,x_j)$}

\State Build train and test sets composed of: $(X^{(i)}_{t^*},0)$ and $(X^{(j)}_{t^*},1)$ 

\State Train a neural network classifier $C_\theta$ on the train set

\State Estimate classification risk, $R(C_\theta)$, using held out test set

\If{$R(C_\theta) > \gamma$}

\State Merge $x_i$ and $x_j$

\EndIf

\EndFor

\end{algorithmic}
\end{algorithm}

Under the assumptions of Theorem~\ref{thm}, this procedure groups together initial states whose marginal trajectory distributions become indistinguishable, thereby identifying states belonging to the same basin.

\subsection{Implementation Details}

Our algorithm requires approximating $k(k-1)/2$ Bayes Classifiers (given a set of $k$ candidate modes). In practice, we train a single neural network to approximate all classifiers simultaneously. We also slightly re-frame the classification problem for simplicity. In our experiments, the neural network takes in two states $x_1$ and $x_2$ and predicts whether or not the two states belong to the same initialization. The label remains to be 0 or 1 denoting whether or not the two states come from the same initialization. We use a simple binary cross entropy loss function. This formulation avoids the need to specify which candidate modes are under comparison at any given time. To estimate the risk of the neural network classifier for any given pair $(i,j)$ of candidate basins, we simply estimate the probability of misclassification from sampling pairs of states from modes $(i,j)$.  

The algorithm requires a set of possible basin candidates along with some hyper parameters including risk threshold for merging candidates, a trajectory length and number of samples for training a neural network and a randomly initialized neural network. The set of basin candidates is simply a set of states, with each state representing the basin which it belongs to. The algorithm determines all such pairings of candidate states that in fact belong to the same underlying basin.

\subsection{Partition Indicator Estimation}

The method we have described so far produces an initial estimate of candidate basins and then refines the estimate by determining sets of candidate basins which in fact belong to the same ground truth basin. The result of running the algorithm is a set of points which are partitioned such that each partition belongs to an underlying basin. A relayed task of interest is estimating indicator functions for each component which can serve as a way to identify new unseen samples with the estimated components/partition. In addition to the partitioning of candidate basin representatives which are returned by our algorithm, we additionally end up with a trained neural network for the classification of states across candidate basins. This trained network and the mapping asserting which candidate basins to merge can be used to assign new unseen states to an underlying basin, thus capable of acting as an indicator function for any of the determined basins.


\section{Experiments}

In this section we demonstrate our method on a handful of applications. These include processes induced by the Metropolis Adjusted Langevin Algorithm (MALA) \citep{roberts1996exponential} on various synthetic energy functions, spherical Stochastic Gradient Descent for solving the Phase Retrieval problem and the molecular dynamics for the Alanine Dipeptide molecule.
With the exception of our molecular dynamics application, our examples cater to the scenario where cross-basin transitions occur rarely if at all and our processes exist in high dimensional, continuous state spaces.
To the best of our knowledge there are no existing methods designed specifically for this regime, however, we attempt to solve the basin identification problem in these cases by experimenting with a number of existing methods in order to establish baseline performance. These include a combination of standard clustering based approaches and MSM based approaches (including VAMPnet \citep{mardt2018vampnets}) which we outline in detail in Appendix \ref{appendix:experiments_baselines}. In contrast to our method, some baselines require pre-specifying the number of basins/clusters to detect, in which case we provide the baseline with the correct number of basins. In all experiments we implement a simple symmetric siamese neural network architecture consisting of dense layers and ReLU activations \citep{glorot2011deep, bromley1993signature, hadsell2006dimensionality}. We evaluate methods using two standard clustering metrics
commonly used in the machine learning literature. With ground truth basin labels available, we measure similarity between
predicted and true partitions using the Adjusted Rand Index (ARI) \cite{hubert1985comparing} and Normalized Mutual Information (NMI) \cite{strehl2002cluster}. These metrics are standard for clustering evaluation and correct for chance agreement between partitions. Additionally, for certain experiments we report the number of predicted basins by the algorithm compared to the ground truth.

\begin{figure}[t!]
\centering

\begin{subfigure}{0.24\textwidth}
    \centering
    \includegraphics[width=\linewidth]{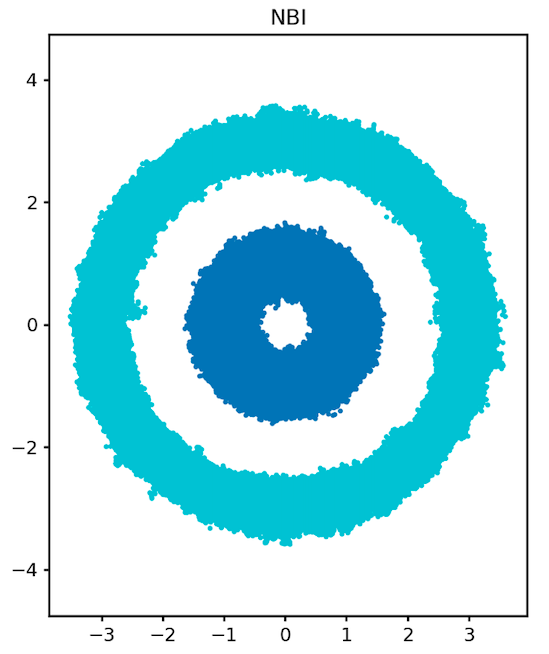}
    \caption{}
\end{subfigure}
\hfill
\begin{subfigure}{0.24\textwidth}
    \centering
    \includegraphics[width=\linewidth]{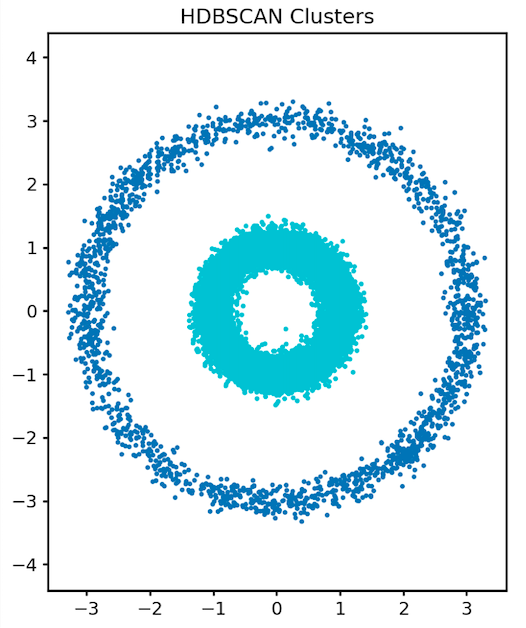}
    \caption{}
\end{subfigure}
\hfill
\begin{subfigure}{0.24\textwidth}
    \centering
    \includegraphics[width=\linewidth]{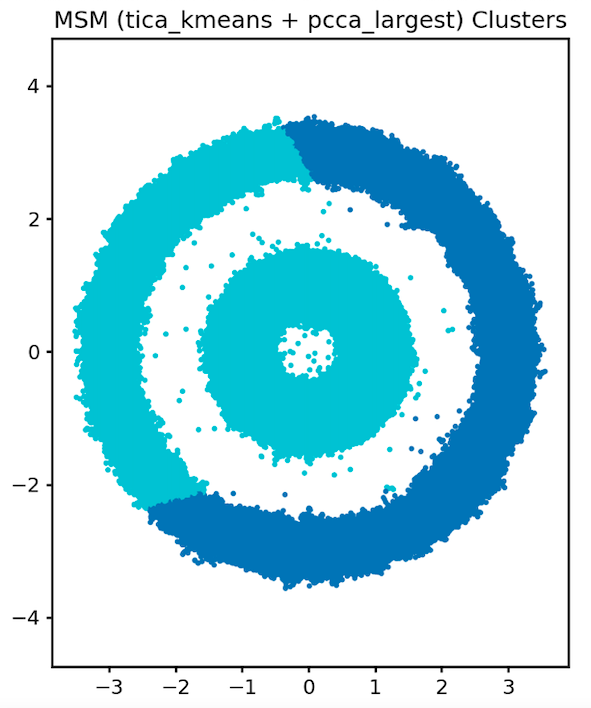}
    \caption{}
\end{subfigure}
\hfill
\begin{subfigure}{0.24\textwidth}
    \centering
    \includegraphics[width=\linewidth]{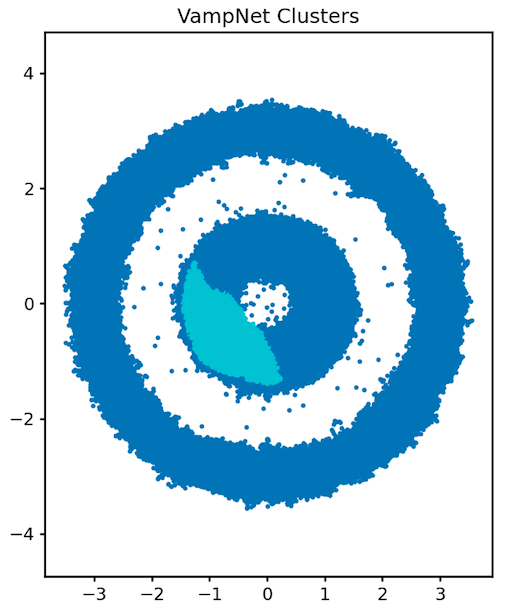}
    \caption{}
\end{subfigure}

\vspace{0.5em}

\begin{subfigure}{0.24\textwidth}
    \centering
    \includegraphics[width=\linewidth]{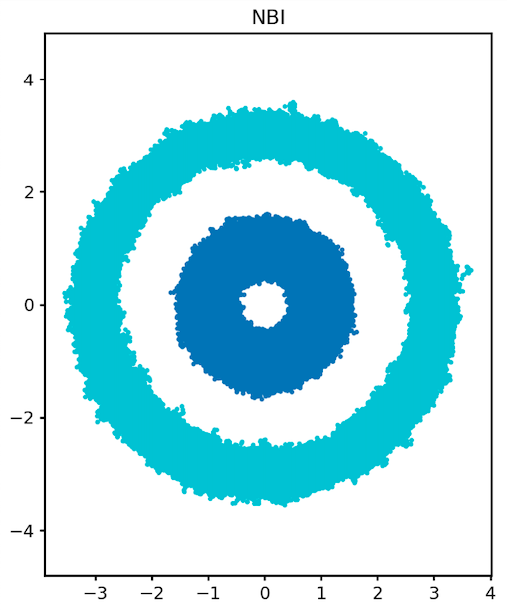}
    \caption{}
\end{subfigure}
\hfill
\begin{subfigure}{0.24\textwidth}
    \centering
    \includegraphics[width=\linewidth]{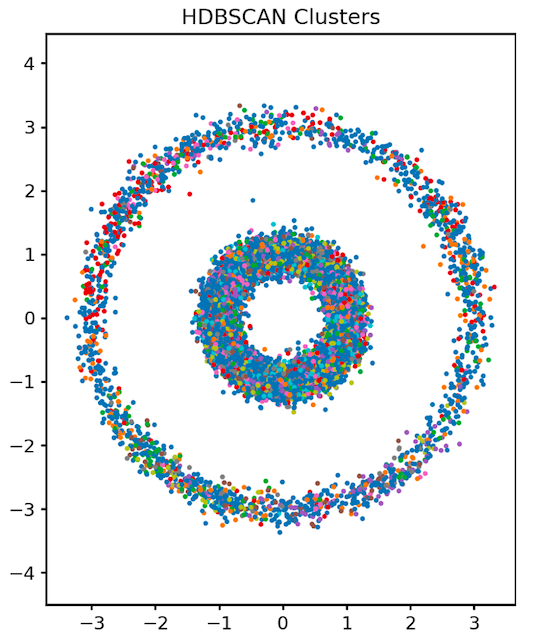}
    \caption{}
\end{subfigure}
\hfill
\begin{subfigure}{0.24\textwidth}
    \centering
    \includegraphics[width=\linewidth]{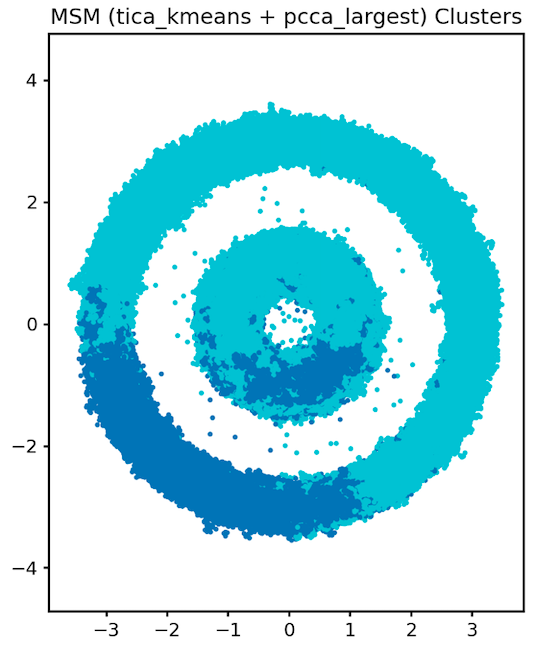}
    \caption{}
\end{subfigure}
\hfill
\begin{subfigure}{0.24\textwidth}
    \centering
    \includegraphics[width=\linewidth]{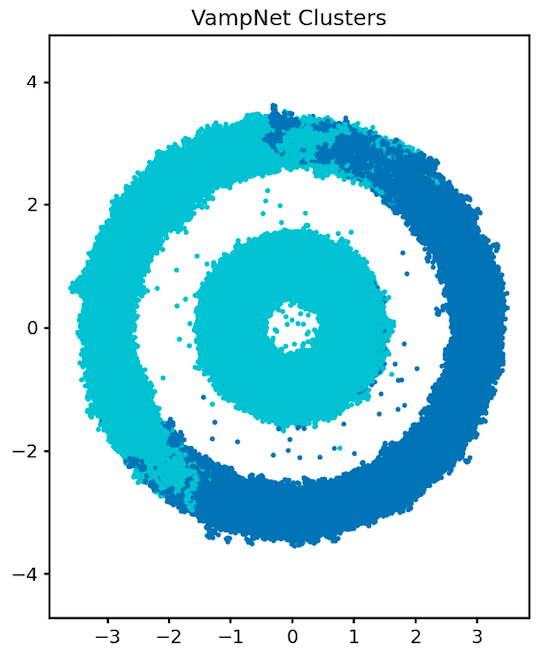}
    \caption{}
\end{subfigure}

\caption{The figure contains trajectories plotted by cluster label for various different methods. The top row corresponds to the MALA process running on a simple 'double ring' energy function composed of two Gaussian rings. The bottom row corresponds to the same structure embedded in 100 dimensional ambient space. In this case, we plot only the low dimensional subspace in which the separability of the basins can be visualized.}
\label{ring_plot}
\end{figure}

\begin{table}[htbp] 
    \centering
    \caption{Reported ARI Across Synthetic Low and (embedded) High Dimensinional Processes. We report the mean and standard deviation across 10 independent experimental results. If a method detects only a single basin, the reported ARI is 0.}
    \label{tab:embedded_results}
    \begin{tabular}{lcccccc} 
        \toprule
        \multirow{2}{*}{\textbf{Method}} & \multicolumn{2}{c}{\textbf{2d Double Ring}} & \multicolumn{2}{c}{\textbf{2d GM}} & \multicolumn{2}{c}{\textbf{3d Helix}} \\
        \cmidrule(lr){2-3} \cmidrule(lr){4-5} \cmidrule(lr){6-7}
        & $D=2$ & $D=100$ & $D=2$ & $D=100$ & $D=3$ & $D=100$ \\
        \midrule
        NBI (Ours)  & \textbf{1.0} $\pm 0.00$& \textbf{1.0} $\pm 0.00$& \textbf{1.0} $\pm 0.00$  & \textbf{1.0} $\pm 0.00$ & \textbf{1.0} $\pm 0.00$& \textbf{0.90} $\pm 0.30$\\
        VampNet  & 0.36$\pm 0.35$ & 0.63 $\pm 0.31$ & 0.68  $\pm 0.21$ & 0.54 $\pm 0.23$ & 0.14 $\pm0.15$&  0.00 $\pm 0.01$\\
        HDBSCAN & \textbf{1.0} $\pm 0.00$ & 0.01 $\pm 0.01$& \textbf{1.0} $\pm 0.00$ & 0.00 $\pm 0.00$ & \textbf{1.0} $\pm 0.00$ &  0.00 $\pm 0.00$\\

        MSM  & 0.92 $\pm 0.11$& 0.22 $\pm 0.19$ &0.75  $\pm 0.38$ & \textbf{1.0} $\pm 0.00$ & 0.04 $\pm 0.03$ & 0.00 $\pm 0.00$ \\

        \bottomrule
    \end{tabular}
    
\end{table}

\subsection{Embedded Low Dimensional Structure}
\label{experiment:low_dim_structure}
We consider 3 synthetic stochastic processes constructed by running the Metropolis Adjusted Langevin Algorithm (MALA) with two and three dimensional energy functions. We demonstrate our approach along with our baselines on these low dimensional processes and additionally we consider high dimensional processes constructed by linearly embedding these low dimensional structures into high dimensional ($d=100$) ambient spaces. Specifically, we construct augmented energy functions which decompose the state space into the low dimensional linear subspace and the orthogonal subspace. The energy of the low dimensional subspace is computed using the original low dimensional energy and the energy of the orthogonal component is that of an isotropic Gaussian. MALA is then run on the augmented energy in ambient space. These simple high dimensional examples represent scenarios where the basin structure is contained on a low-dimensional manifold which is embedded in a noisy high dimensional ambient space. We outline the energy functions in more detail in Appendix \ref{appendix:experiments_energies}. Results are reported in table \ref{tab:embedded_results} (and further in \ref{appendix:embedded_results}, along with a visualization of some produced clusters in figure \ref{ring_plot}. We see that some of the baselines achieve high performance across one or two of the tasks but in general have very inconsistent performance. Additionally, we observe that for many of the methods performance may be strong on the low dimensional process but degrades significantly when applied to the exact same structure embedded in a noisy high dimensional ambient space. Thus highlighting a clear limitation of existing approaches. In contrast, our method achieves consistent high performance across all tasks demonstrating a robustness to operating in noisy high dimensional spaces.

 We suspect a few reasons for why baselines struggle. Firstly, samples come from trajectories and hence display strong auto-correlation, which for HDBSCAN can affect stability of connected components. Additionally, both HDBSCAN and MSM make use of Euclidean distance to determine connectedness and cluster assignments. This may be misleading as densities may have disconnected components which are close in ambient space, and these methods are unable to consider the underlying geometry of the density. We comment more on this in appendix \ref{appendix:experiments_failure} 

\subsection{SGD on Phase Retrieval}

Phase retrieval is a classical problem which can be solved by gradient descent and whose dynamics are well understood \citep{ji2016gradient, tan2019online, arous2021online, mignacco2021stochasticity}. The problem is to recover an unknown signal $x^\star \in \mathbb{S}^{d}$ (the $d$-dimensional sphere) from iid observations of the form $y_i = (a_i^\top x^\star)^2 + \epsilon_i$, with $a_i \sim \mathcal{N}(0, I_d)$ and $\mathbb{E}\epsilon=0$. The unknown vector can be estimated by $x$ obtained via spherical stochastic gradient descent (SGD) on the loss: $L(x, (a_i, y_i)) = [(a_i^Tx)^2 - y_i]^2$. This problem exhibits an inherent sign ambiguity, with two global minima at $\pm x^\star$. When optimized via SGD, the resulting dynamics evolve in a high dimensional parameter space but are effectively governed by a single summary statistic, the alignment with $x^\star$. We thus take SGD on the estimated parameter space (with dimension $d=200$) as our Markov process and attempt to recover the two known basins. Results are reported in table \ref{experiments:sgd}. 

\begin{table}[h]
\centering
\caption{Results for SGD on Phase Retrieval.}
\label{experiments:sgd}

\begin{tabular}{lccccc}
\toprule
 \textbf{Metric} & \textbf{NBI (ours)} & \textbf{VAMPnet} & \textbf{HDBSCAN} & \textbf{MSM}  \\ \midrule
ARI    & \textbf{0.88} $\pm 0.30$ & 0.57 $\pm 0.46$              & 0.43 $\pm 0.03$  &  0.86 $\pm 0.03$  \\
NMI & \textbf{0.87} $\pm 0.30$ & 0.55 $\pm 0.45$  & 0.38 $\pm 0.02$  &  0.78 $\pm 0.03$  \\ \bottomrule
\end{tabular}
\end{table}

\subsection{Gaussian Mixture Models}

A standard example of a high dimensional energy landscape with multiple well separated basins is a high dimensional Gaussian Mixture. We test our method along with a strong performing baseline of HDBSCAN which also does not require prior knowledge on the number of components. To construct these landscapes we generate some arbitrary number of components at random locations. The purpose of this experiment is simply to demonstrate that our method can handle processes with many disjoint basins without specifying the number of basins. For this experiment we report ARI, NMI and the predicted number of basins which can be compared to the ground truth. More details and results can be found in Appendix \ref{appendix:experiments_energies} and table \ref{appendix:gmm_table}.

\subsection{Alanine Dipeptide}

Alanine dipeptide is a widely studied molecular dynamics system (of 22 particles represented by a 66 dimensional state vector) whose slow dynamics are well captured by two backbone dihedral angles, $\phi, \psi$, commonly visualized via the Ramachandran plot. The resulting low-dimensional representation exhibits multiple metastable basins. However, existing works suggest that the separation of time scales between intra-basin-mixing and cross-basin transitions is not that strong with cross-basin transitions occurring with non-trivial frequency even across short trajectories. Thus the problem assumptions under which we expect success from our method may not be valid. Our goal for this experiment is not to demonstrate superiority over existing methods but rather to show that our method may provide an additional tool for those studying molecular dynamics and additionally to demonstrate our method on a real world system. In particular, a system that exists in a high dimensional continuous state space (66 dimensions) and has known metastable structure existing on a low dimensional manifold within this space. We apply our method to both the full 66-dimensional system and additionally to the exact same trajectories projected into the 2-d subspace of the backbone dihedral angles. We provide visualization of the Ramachandran plot and the classification of trajectories by our model done in both the full and reduced state systems. The model trained in the full high dimensional space over-segments the trajectories in the Ramachandran, whereas the model trained in 2 dimensions draws a clear boundary between 2 basins. This suggests that while the distributions of the trajectories under this 2-d projections may have become indistinguishable, their full 66-dimensional distributions remain distinguishable, providing some insight into the mixing timescales of the full and low dimensional spaces. These figures along with additional discussion can be found in Appendix \ref{appendix:alanine_dipeptide}.

\section{Conclusion, Limitations and Future Work}
We propose a method for identifying metastable basins of a time-homogeneous Markov process from black-box access to its dynamics. Our focus is on high dimensional continuous state spaces where transitions between basins occur rarely or not at all (i.e., the process may be reducible) and further, within-basin mixing occurs on shorter timescales. 
Experiments on several synthetic systems demonstrate that our approach performs well in high dimensional and noisy settings.

The proposed method has some notable limitations. It assumes a separation between within-basin and global mixing timescales.
When this separation is weak or the timescale is poorly chosen, the basin definition considered here may become ambiguous or sensitive to hyper parameters including the merging threshold. In practice, we found monitoring true and false positive rates during training helpful for determining useful hyperparameters. 

Future work includes exploring applications in areas where high dimensional Markov processes arise, such as neural sampling methods, reinforcement learning, molecular dynamics, loss landscape analysis, etc. As well as adapting the method to settings where only trajectory data are available.

\newpage
\bibliographystyle{plain}
\bibliography{main}

\newpage
\appendix

\section{Proof of Main Result}
\label{appendix:thm_proof}
Below we provide the proof of Theorem \ref{thm}.
\begin{proof}
    The proof is straight forward. Suppose WLOG $x_1 \in c_w$. We start with 1. Note that by definition of the Basin Partition, we have $\mathbb{P}(\tau_{w} > t^*) \geq 1-\delta$. We can simply bound:
    
    \begin{align*}
        R(C^{Bayes,t^*}_{(x_0, x_1)}) &\geq \mathbb{P}(C^{Bayes,t^*}_{(x_0, x_1)}(X^{(y)}_{t^*})\neq y, \tau_{w}>t^*)\\& = \mathbb{P}(C^{Bayes,t^*}_{(x_0, x_1)}(X^{(y)}_{t^*})\neq y\,| \tau_{w}>t^*)\mathbb{P}(\tau_{w}>t^*)\\&\geq(1-\delta)\mathbb{P}(C^{Bayes,t^*}_{(x_0, x_1)}(X^{(y)}_{t^*})\neq y\,| \tau_{w}>t^*)
    \end{align*}

\[
\sup_{x\in c_w}
\left\|
P_t(\cdot \mid \tau_{w} > t, X_0=x)
- \pi_w
\right\|_{TV}
\le \varepsilon
\quad \text{for all } t \ge t^* .
\]

    Now by assumption we have that $\|P_t(X^{(y)}_{t^*}|\tau_{w}>t^*) - \pi_w\|_{TV} \leq \epsilon$ whether $y=0$ or $y=1$. By triangle inequality we thus have $\|P_t(X^{(0)}_{t^*}|\tau_{w}>t^*) - P_t(X^{(1)}_{t^*}|\tau_{w}>t^*)\|_{TV}\leq 2\epsilon$. It is a well known fact that the risk of a Bayes classifier which selects which of two distributions an observable came from, is equal to a linear function of the total variation distance between the probability measures. Here we get an inequality since the we consider the conditional risk of the Bayes classifier in contrast to the risk of the Bayes classifier for the conditional distribution: $$\mathbb{P}(C^{Bayes,t^*}_{(x_0, x_1)}(X^{(y)}_{t^*})\neq y\,| \tau_{w}>t^*) \geq \frac{1}{2}(1-\|P_t(X^{(0)}_{t^*}|\tau_{w}>t^*) - P_t(X^{(1)}_{t^*}|\tau_{w}>t^*)\|_{TV}) \geq \frac{1}{2}(1-2\epsilon)$$ Putting it together we get the desired: $$R(C^{Bayes,t^*}_{(x_0, x_1)}) \geq (1-\delta)(\frac{1}{2}-\epsilon)$$

    Moving on to 2. Consider the following classifier: $C(x) = 1$ if $x\in w'$ and 0 otherwise. Then
    
    \begin{align*}
        R(C) &= \mathbb{P}(y=0, X^{(0)}_{t^*}\in w') + \mathbb{P}(y=1,X^{(1)}_{t^*}\notin w') \\ & \leq \mathbb{P}(y=0, \tau_w < T) + \mathbb{P}(y=1,\tau_{w'} < T)
        \\ & = \mathbb{P}(y=0)\mathbb{P}(\tau_w < T |y=0) + \mathbb{P}(y=1)\mathbb{P}(\tau_{w'} < T |y=1)\\ &\leq \frac{1}{2}\delta + \frac{1}{2}\delta \\& = \delta
    \end{align*}
    
    Using the well known fact that $R(C_{Bayes}) \leq R(C)$ for any classifier $C$ completes the proof. 
\end{proof}

\section{Experiments}

In this section we include additional experimental details for the sake of clarity. Additionally, we have made our code publicly available for the sake of reproducing our results and for future work. All experiments are run on a NVIDIA L40S GPU. Any experimental run can be performed on this machine in a matter of minutes up to an hour. For the first two experiments under which we report quantitative metrics, we repeat each experimental run 10 times under different seeds. The total GPU time to run all experiments amounts to approximately 1.5 days.

\subsection{Baselines}\label{appendix:experiments_baselines}

We experiment with numerous baselines. We emphasize that none of the methods we experiment with were designed specifically for the scenarios we apply them to here, nevertheless we do our best to adapt these methods appropriately. In section \ref{appendix:experiments_failure} we provide discussion on the failure modes of these baselines which ultimately motivate the necessity of our proposed method. In our experiments we roughly equate the number of trajectories sampled in an attempt to keep comparisons fair. When applying methods however that have quadratic complexity in the number of samples we downsample to 15,000 samples for computational feasibility. 

\subsubsection{Clustering}

One way of framing the problem of identifying metastable components is in terms of unsupervised clustering of simulated trajectories. Many clustering algorithms exist and we experiment with a few of these methods both on their own, as well as integrated into MSMs (discussed below). Specifically we consider HDBSCAN \citep{campello2015hierarchical} and Diffusion Maps \citep{coifman2006diffusion} for clustering raw simulated trajectories. For each experiment, we randomly initialize and simulate some fixed number of trajectories. We then treat each state across each trajectory as a sample to construct a 'dataset'. Both of these methods require computing pairwise distances and thus have quadratic complexity in the number of samples and thus we only perform clustering on a random subsample of the constructed dataset. In our experiments we produce a subsample of 15,000 for clustering, this is why the cluster plots appear more sparse for these algorithms. Despite utilizing less data, in some cases subsampling can improve the performance of HDBSCAN by reducing auto-correlation. We found Diffusion Maps to perform poorly in all our experiments and thus refrain from reporting their results. 

\subsubsection{VAMPnet}

VAMPnets \cite{mardt2018vampnets} are a neural network-based approach for learning low-dimensional representations of dynamical systems by approximating the leading singular functions of the transfer operator \cite{mardt2018vampnets}. Given time-lagged pairs of samples $(x_t, x_{t+\tau})$, a VAMPnet consists of two (typically identical) neural networks that map configurations to a feature space. The network is trained to maximize a variational score derived from the \emph{Variational Approach for Markov Processes} (VAMP), which encourages the learned features to capture the slowest dynamical modes of the system.

Concretely, letting $f_\theta(x)$ denote the network output, the method constructs covariance and cross-covariance matrices between $f_\theta(x_t)$ and $f_\theta(x_{t+\tau})$, and optimizes a matrix norm (e.g., the VAMP-2 score) that approximates the dominant singular values of the Koopman operator. The learned representation can then be used for clustering, dimensionality reduction, or as input to Markov state model construction. Unlike classical pipelines that rely on hand-crafted features and linear methods such as TICA, VAMPnets learn nonlinear transformations directly from data in an end-to-end fashion.

We experiment with VAMPnets as a baseline in a few manners and ultimately only report results for one variant. The output of a VAMPnet is a softmax classification over some specified dimension, $c$ meaning that the network acts as softly assigning any given state to one of $c$ classes. As one baseline, we train a VAMPnet on randomly sampled trajectories and provide the true number of basins as the value of $c$. We then take the softmax assignment of each state as the cluster label for that state. This is the 'VAMPnet' method we report. We additionally experiment with using the second last layer of the VAMPnet as an embedding which captures slow dynamics and apply clustering algorithms (K-Means and HDBSCAN) on the embedded samples. We do not find the performance of these approaches to improve upon the VAMPnet or basic HDBSCAN methods and thus do not report the results. We emphasize that VAMPnets are designed for the scenario where transitions between metastable basins are slow but do indeed occur and thus we are not surprised that they are not always effective as a baseline for regimes where we do not observe cross-basin transitions.

\subsubsection{MSMs}

Markov State Models (MSMs) are a widely used framework for modeling the long-timescale dynamics of molecular systems by discretizing the state space into a finite set of metastable states \cite{prinz2011markov}. Given a trajectory $\{x_t\}$, the configuration space is first partitioned into discrete states (e.g., via clustering in some feature space), and transitions between these states are counted at a fixed lag time $\tau$ to estimate a transition probability matrix.

Under the Markov assumption at lag time $\tau$, the dynamics are approximated by a discrete-time Markov chain, where the transition matrix encodes the probabilities of moving between states over time $\tau$. Spectral analysis of this matrix yields estimates of slow dynamical processes, including metastable states and their associated relaxation timescales. In practice, MSM construction often involves a pipeline of feature selection, dimensionality reduction (e.g., via time-lagged independent component analysis), clustering, and validation to ensure approximate Markovianity at the chosen lag time.

We experiment with multiple MSM variants. One design choice for an MSM is identifying how to perform a discretization / compute a set of 'microstates' on which to compute the transition matrix. A second design choice is how to extract basins or 'macrostates' based on the transition matrix over microstates. A common approach for constructing microstates in Markov state models is to first apply time-lagged independent component analysis (TICA) to project high dimensional molecular configurations onto a low-dimensional space that captures the slowest dynamical modes, and then perform k-means clustering in this reduced space to discretize the data into microstates. PCCA+ (Perron Cluster Cluster Analysis) \citep{deuflhard2005robust} is a standard method for coarse-graining a Markov state model by grouping microstates into metastable macrostates. Given a microstate transition matrix, PCCA+ uses the dominant eigenvectors of the matrix to identify almost-invariant sets and constructs a soft assignment of microstates to macrostates, which can then be converted into a hard clustering. This allows one to extract a small number of kinetically meaningful metastable states directly from the MSM. As our MSM baseline we opt for TICA + Kmeans for microstate selection and PCCA+ for microstate to macrostate assignment, ultimately yielding a complete pipeline. 

We note that this pipeline (and MSMs more generally) are primarily used for detecting metastable components in the regime where cross-basin transitions are rare but occurring in observed trajectory data. Additionally MSMs are applied directly to observed trajectory data and so to use them as baselines here we randomly sample some number of trajectories from the given dynamics before applying the MSM. We also experiment with some different choices for constructing microstates and mapping microstates to macrostates. For the former, we consider defining clusters by using the same set of 'candidate modes' used by our method in place of the cluster centroids of TICA + K-means. We also experiment with mapping each state to a microstate by using a trained VAMPnet to produce soft assignments to some specified number of microstates. For assigning microstates to macrostates we also experiment with using connected components of the microstates. Since we consider problems where state spaces may be completely reducible or cross-basin transitions may not even occur on the time scales under which we simulate dynamics, we allow for the possibility that the microstate transition matrix be decomposed into disjoint connected components and use these components as macrostates. However we find that all of these alternatives results in much worse average performance than the canonical MSM pipeline and thus we do not report those results.

\subsection{Energy Functions}\label{appendix:experiments_energies}

\subsubsection{Low-dim Energy Functions}

For our first experiment we consider three low-dimensional energy functions. We apply our method and the baselines directly to the MALA process constructed with these low-dimensional energy functions and additionally to augmented energy functions consisting of the same low dimensional structure linearly embedded in a high dimensional state space with isotropic Gaussian energy in the orthogonal directions. The first energy function consists of two rings, with one ring existing inside of the other. The second energy function is a 2-d Gaussian mixture with cluster locations randomly selected. We fix the cluster locations (in 2-d space) across the low and high dimensional experiments. And lastly a 3-d helix which consists of 2 basins wrapping around eachother and 4 Gaussians at each end point. We provide a plot of samples produced by each of the low dimensional processes in figure \ref{appendix:process_plot}.

\begin{figure}[htbp]
    \centering
    
    \begin{subfigure}[b]{0.4\textwidth}
        \centering
        \includegraphics[width=\textwidth]{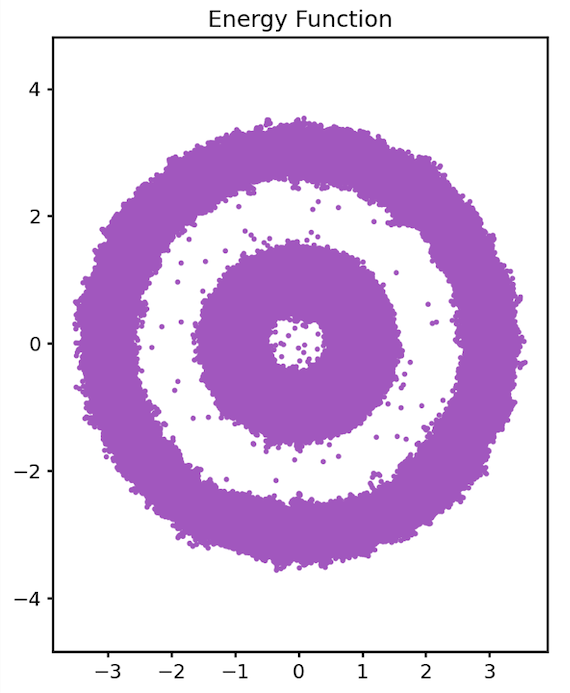}
        \caption{2d Double Ring Energy}
    \end{subfigure}
    \hfill 
    \begin{subfigure}[b]{0.4\textwidth}
        \centering
        \includegraphics[width=\textwidth]{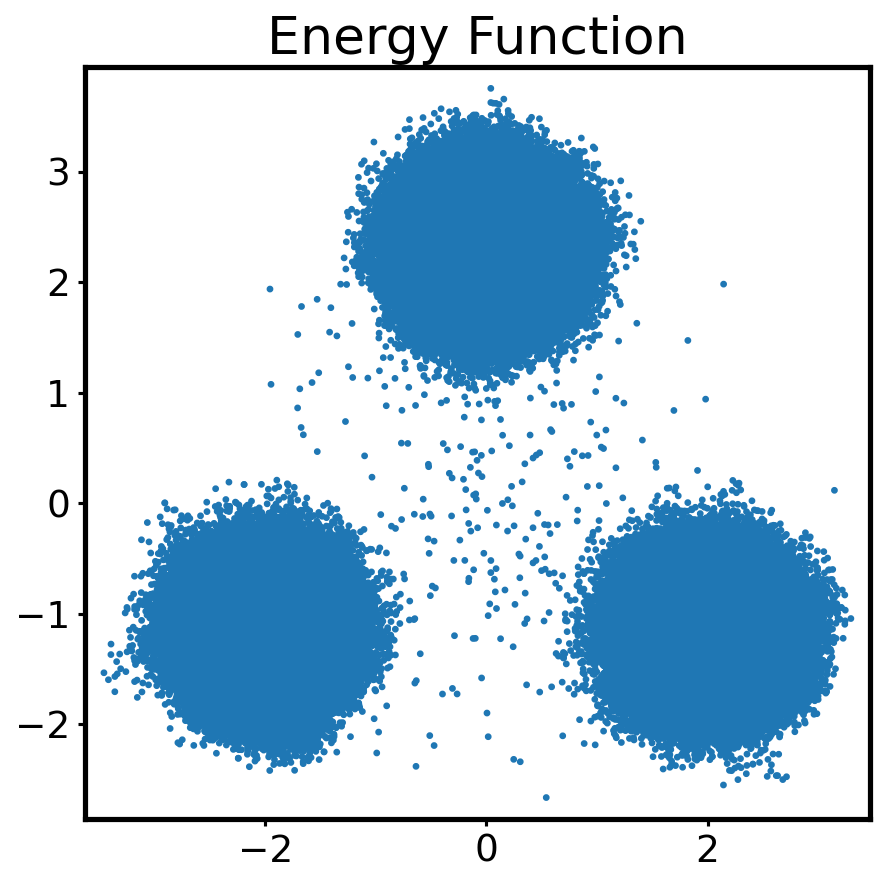}
        \caption{2d Gaussian Mixture Energy}
    \end{subfigure}

    \vspace{0.8cm} 

    \begin{subfigure}[b]{0.8\textwidth}
        \centering
        \includegraphics[width=\textwidth]{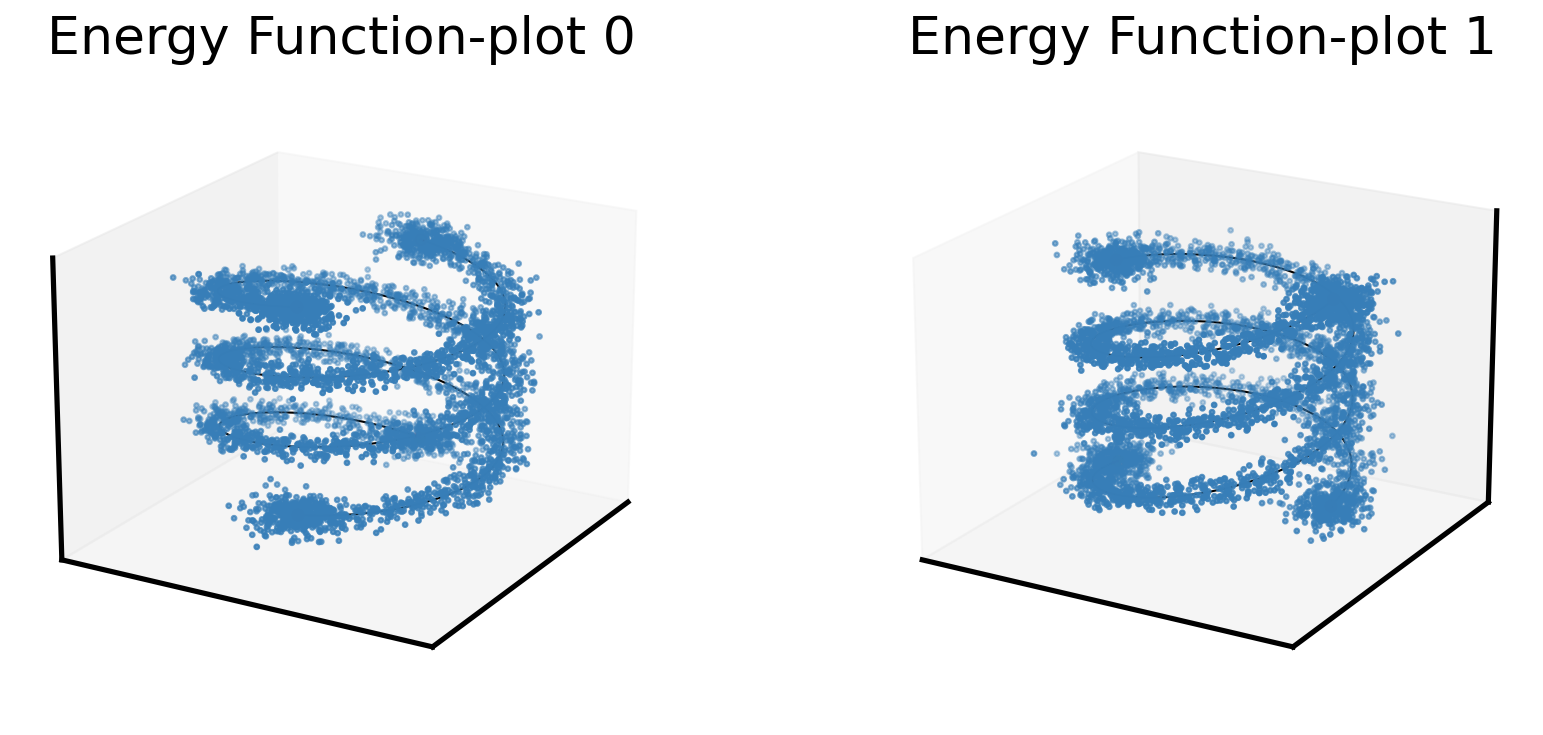}
        \caption{3d Helix Energy}
    \end{subfigure}

    \caption{visualizations of samples drawn from the low dimensional stationary distributions of the MALA processes used in experiments.}
    \label{appendix:process_plot}
\end{figure}

\subsubsection{Gaussian Mixtures}

The second experiment we consider is simple Gaussian mixtures in 100 dimensional space. We select some arbitrary number of components between 0 and 50 and the locations for these components are selected randomly from the sphere of radius 10 (i.e. $\sqrt{d}$). This ensures with high probability the components remain well separated and we can expect successful methods to identify the correct number of components as well as having high NMI and ARI. For the discovery stage of our algorithm we randomly initialize 100 MALA chains with a standard isotropic Gaussian. We note that in some cases the discovery stage of our algorithm fails to identify all the components leading to suboptimal performance. This could be corrected for by producing more random trajectories during the sampling stage. In contrast the baseline method HDBSCAN which is also successful in this setup has a higher recall for mode discovery because all sampled trajectories come from distinct random initializations, whereas under our method, the majority of the sampling budget is spent producing trajectories from the candidate basins after the discovery stage. This highlights a limitation of our method in terms of reliance on the method of basin discovery. The results in terms of number of basins discovered can be found in table
\ref{appendix:gmm_table}. In this experiment we fix the compute allocated even as the number of components scales up which we observe harms the performance of our method.

\begin{table}[h]
\centering
\caption{Number of basins identified across various Gaussian mixtures in 100 dimensional space. GMM X-Basin corresponds to MALA run on a Gaussian mixture with X components.}
\label{appendix:gmm_table}
\begin{tabular}{lcccc}
\toprule
\textbf{Method} & \textbf{GMM 7-Basin} & \textbf{GMM 12-Basin} & \textbf{GMM 26-Basin} & \textbf{GMM 42-Basin} \\ \midrule

NBI (Ours) &  7$\pm0.0$  & 12 $\pm0.0$ &  25.1$\pm 0.54$  &  37.2 $\pm1.66$  \\
HDBSCAN &  7$\pm0.0$ & 12$\pm0.0$  &  26 $\pm 1.66$  &  41.8 $\pm 0.4$  \\ \bottomrule
\end{tabular}
\end{table}

\subsection{Failure Modes}\label{appendix:experiments_failure}

Our first experiment which involves embedding a low-dimensional basin structure into a high dimensional noisy ambient space highlights a shortcoming of existing approaches at detecting long-term meta stable basins or even reducible components. These synthetic processes are representative of a broader class of processes which exist in a high dimensional state space, but who's basin-defining structure exists on a low dimensional manifold. Further this low-dimensional manifold may have its own non-Euclidean geometry. Below we discuss possible reasons why this class of processes may pose problems for existing methods.

\subsubsection{HDBSCAN}

Firstly we consider HDBSCAN, a commonly used and powerful density based clustering algorithm. This method builds clusters based on connected components of a constructed distance graph. The ability to cluster samples based on connectedness is favorable as we essentially want to detect the continuous state space analog of connected components. However, there are a few possible limitations which can restrict this approach. For one, when applied to trajectory data, there is strong auto-correlation present in the samples. HDBSCAN builds clusters based on a stability metric for connected components. This stability metric may become biased when subsets of the samples exhibit auto-correlation and can make the algorithm more inclined to clustering each trajectory individually, something that we observed over the course of experimentation. The results can also be sensitive to stepsize of the process dynamics as this influences the degree of auto-correlation. 

Secondly, the distance between points is computed using Euclidean distance. This can be an unreliable measure of distance between points when the low energy regions of the state space fall on a low dimensional manifold. For example, two points can be close in Euclidean space but may exist in disjoint reducible components of the state space. Further, in a high dimensional space where many directions do not contribute to basin structure, the Euclidean distance can be dominated by noisy irrelevant directions rendering it ineffective. 

\subsubsection{MSMs}

Markov State Models rely on reducing a continuous state space to a finite set of states and assigning each points in the true state space to one of a few candidates. This assignment is often performed based on Euclidean distance. Thus, the assignment can often be problematic for the same reasons as discussed above. This means that the continuous state space is already restricted in terms of which regions can be determined to belong to the same basin or not and this restriction is based on Euclidean distance rather than the geometry of the energy landscape.

\subsubsection{VAMPnet}

In contrast to the two approaches discussed above, VAMPnets do not make use of Euclidean distance when determining assignments from continuous state space to a finite set of states. Like our own method, VAMPnets learn these assignments through the use of neural networks. However, VAMPnets are trained on pairs of lagged states across trajectories and optimize the 'VAMP' objective which measure how well mapped states capture the slow dynamical processes of the system. However, since the method is trained on lagged pairs from trajectories, the method requires observing some cross-basin transitions in order to determine that transitions between those basins are 'slow moving' and without observing cross-basin transitions the model is not incentivized to enforce states belonging to distinct basins to remain distinct after mapping them to the finite set of states. This highlights that VAMPnets are not designed for the set up we consider here and we emphasize that our method is not designed to compete with VAMPnets but offer an alternative tool for solving a related but different problem.

\subsection{SGD on Phase Retrieval}

\begin{table}[h]
\centering
\caption{Results for SGD on Phase Retrieval.}
\begin{tabular}{lccccc}
\toprule
 \textbf{Metric} & \textbf{NBI (ours)} & \textbf{VAMPnet} & \textbf{HDBSCAN} & \textbf{MSM}  \\ \midrule
ARI    & \textbf{0.88} $\pm 0.30$ & 0.57 $\pm 0.46$              & 0.43 $\pm 0.03$  &  0.86 $\pm 0.03$  \\
NMI & \textbf{0.87} $\pm 0.30$ & 0.55 $\pm 0.45$  & 0.38 $\pm 0.02$  &  0.78 $\pm 0.03$  \\ \bottomrule
\end{tabular}
\end{table}

As we can see from the results above, our method is the best performing. The results are somewhat close with MSM and we highlight that the variance of our method is in fact much higher than that of MSM. The results reported are based on 10 independent runs and our method performed near perfectly on 9/10 of the runs, and on one it reported only a single basin (in contrast to the correct response of 2 basins) and hence received an NMI and ARI of 0 in this case. The MSM method in contrast is told that the correct number of basins is two, apriori but still performs worse on average. Further we point out that the MSM method performed reasonably well here as well as on the 2d Gaussian Mixture energy (in the low dimensional Structure embedded in high dimensional space example). Both of these examples the assignment of states in the low dimensional subspace can in fact be determined by Euclidean proximity, which is not the case in the other low dimensional examples, which is why performance of MSM is much worse in those settings.

\subsection{Alanine Dipeptide}
\label{appendix:alanine_dipeptide}
We apply our method to the Alanine Dipeptide molecular dynamics process. We make use of the open source python library OpenMM \citep{eastman2023openmm} for performing simulations. For the discovery phase we initialize 20 independent simulations with a high temperature of 2000 and run for 1000 steps of stepsize 0.002 picoseconds in order to discover candidate basins. For sampling from the discovered candidate basins we reduce the temperature to 200 and run for 700 additional steps. Even on these relatively short timescales, we frequently observe transitions across learned basins suggesting our assumption about separation of inter and intra basin mixing times is not well grounded. Due to this we find that the results of the model can be sensitive to changes in hyper parameters. We train a model for detecting meta stable basins on both the full 66-dimensional trajectories and additionally on the 2-d system that arises from projecting the full dynamics onto the 2 backbone dihedral angles as is done when producing the Ramachandran plot. We plot the trajectories in the Ramachandran plot, colored by basin assignment from both models in figure \ref{alanine_dipeptide_plot}. Despite cross-basin transitions occurring over the course of the simulation, our method identifies 2 basins when trained on the 2-d system and we additionally visualize the partition of the 2-d state space in the same figure referenced above. In contrast, when trained on the full 66 dimensional system, the model identifies 16 out of a possible 20 basins. This implies the trained model in the full 66 dimensional space is able to distinguish the distributions of these short trajectories with distinct initializations when given access to the full system, but is not able to distinguish their distributions when given the 2-d system. This highlights the limitations of our method when the assumption on separation of mixing times is not well founded.

\begin{figure}[h!]
    \centering
    \begin{subfigure}[b]{0.45\textwidth}
        \centering    \includegraphics[width=\textwidth]{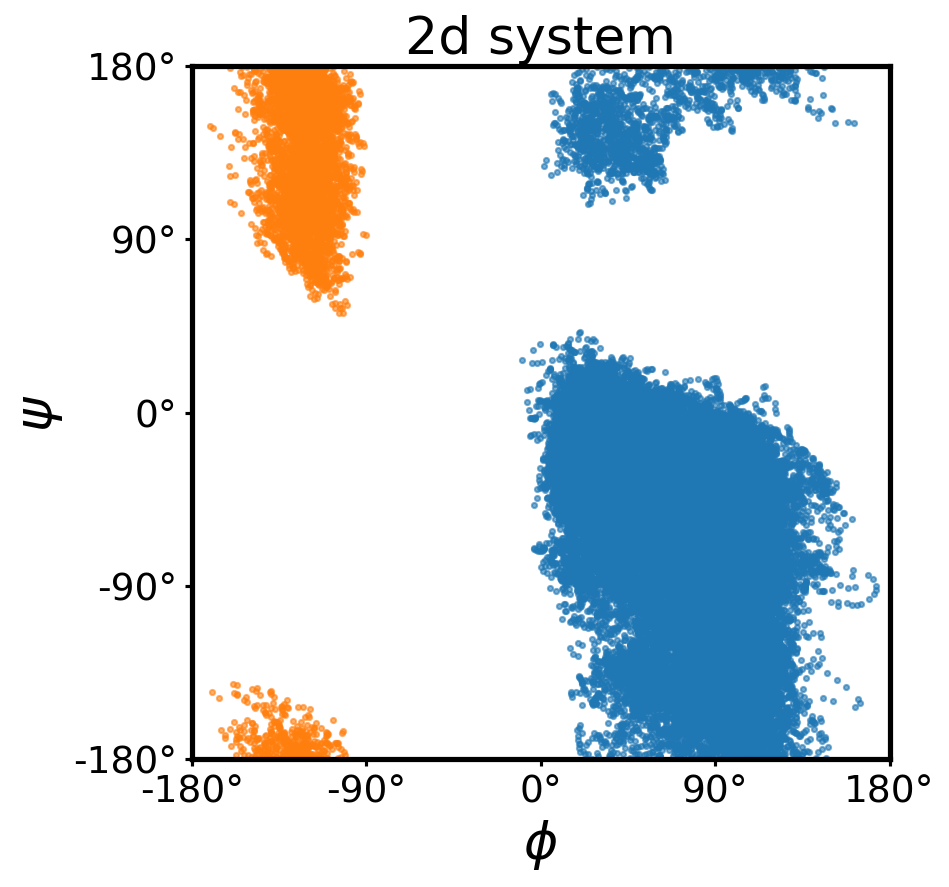}
        \caption{}
    \end{subfigure}
    \hfill
    \begin{subfigure}[b]{0.45\textwidth}
        \centering
        \includegraphics[width=\textwidth]{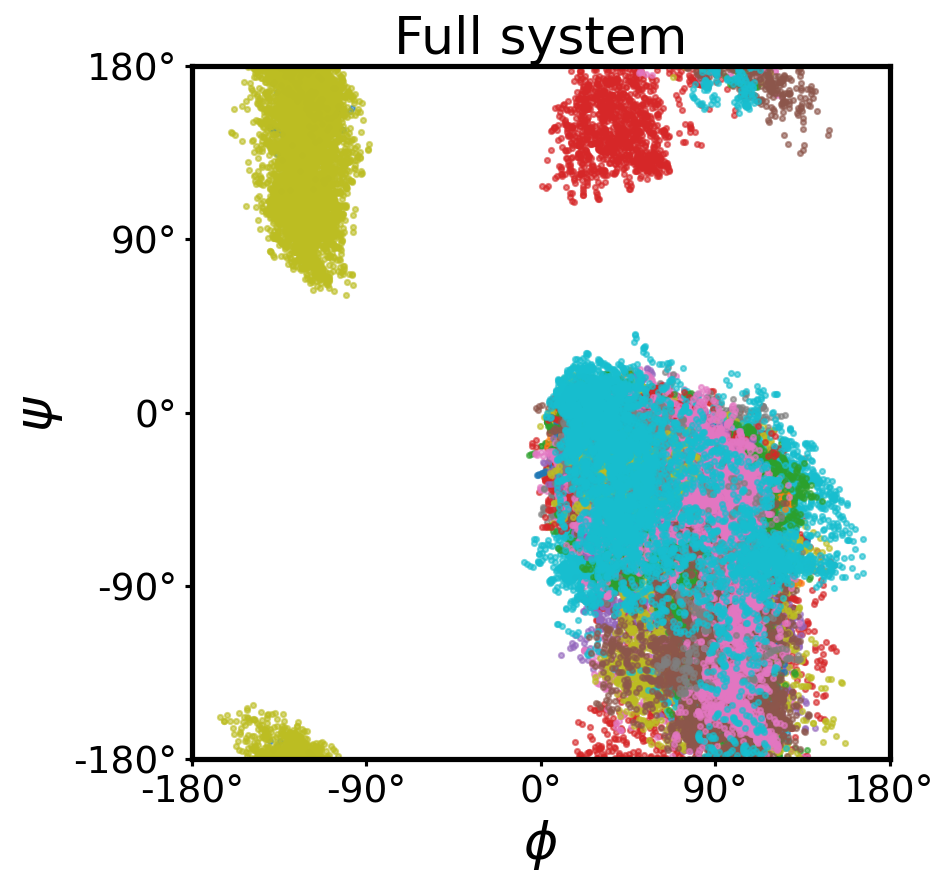}
        \caption{}
    \end{subfigure}

    \vspace{1cm} 

    \begin{subfigure}[b]{0.5\textwidth}
        \centering
        \includegraphics[width=\textwidth]{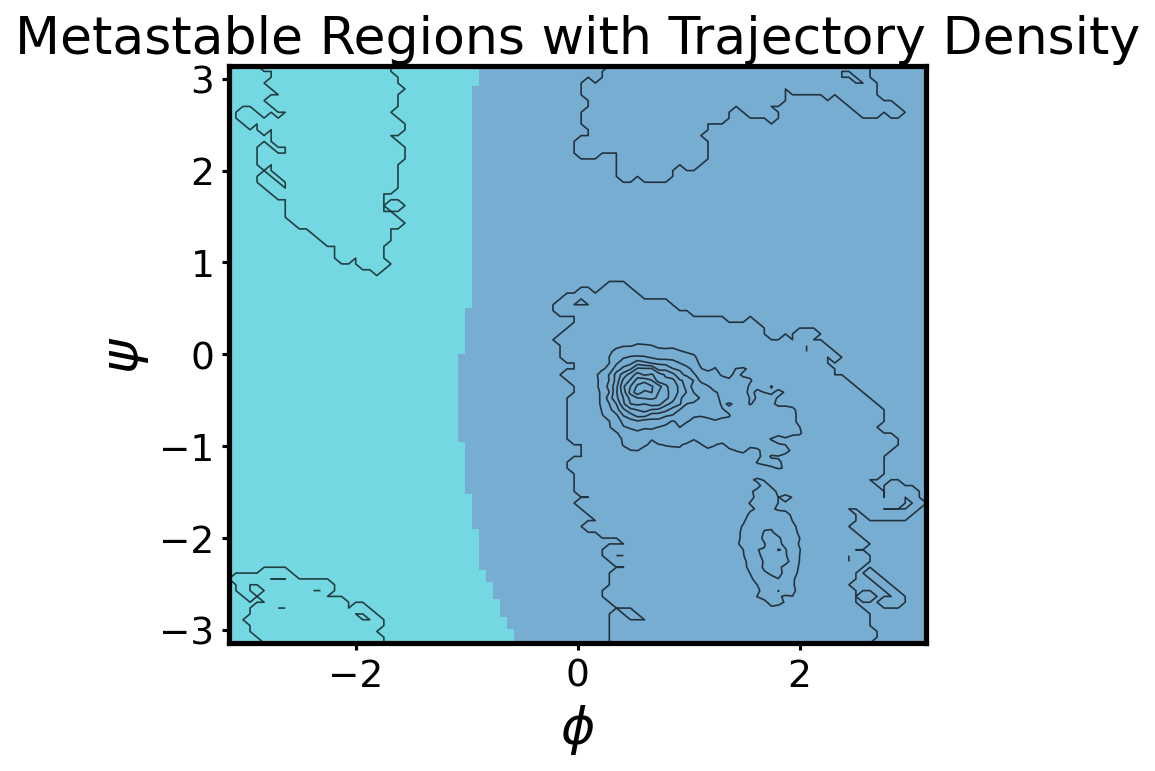}
        \caption{}
    \end{subfigure}
    \hfill
    \begin{minipage}[b]{0.4\textwidth}
        \centering
        \caption{The plots a) and b) display trajectories colored by their assigned basin from the trained NBI. Plot a) corresponds to a model trained only on the 2-d projection of the trajectories whereas plot b) corresponds to clustering according to a model trained on the full 66 dimensional system. Plot c) displays contours for the density of the trajectories plotted above it as well as a coloring denoting the partitioning of the space according to basins learned by the model trained on the 2d system.}
        \label{alanine_dipeptide_plot}
        \vspace{1cm} 
    \end{minipage}
\end{figure}

\subsection{Additional Tables and Figures}

\vspace{6mm}

\centering
    \label{appendix:embedded_results}
    \begin{tabular}{lcccccc} 
        \toprule
        \multirow{2}{*}{\textbf{Method}} & \multicolumn{2}{c}{\textbf{2d Double Ring}} & \multicolumn{2}{c}{\textbf{2d GM}} & \multicolumn{2}{c}{\textbf{3d Helix}} \\
        \cmidrule(lr){2-3} \cmidrule(lr){4-5} \cmidrule(lr){6-7}
        & $D=2$ & $D=100$ & $D=2$ & $D=100$ & $D=3$ & $D=100$ \\
        \midrule
        NBI (Ours)   & \textbf{1.0} $\pm 0.0$ & \textbf{1.0} $\pm 0.00$& \textbf{1.0} $\pm 0.00$ & \textbf{1.0}  $\pm 0.00$ & \textbf{1.0} $\pm 0.00$& \textbf{0.90} $\pm 0.30$\\
        VampNet & 0.33 $\pm0.31 $ & 0.58 $\pm 0.31$ & 0.80 $\pm 0.12$ & 0.66 $\pm 0.24$& 0.18 $\pm 0.17 $& 0.01 $\pm 0.01$\\
        HDBSCAN & \textbf{1.0} $\pm 0.00$& 0.10 $\pm 0.01$& \textbf{1.0} $\pm 0.00$ & 0.00  $\pm 0.00$ & \textbf{1.0} $\pm 0.00$& 0.13 $\pm 0.00$\\
        MSM  & 0.88 $\pm 0.15$ & 0.14 $\pm 0.13$ & 0.79  $\pm 0.33$ & \textbf{1.0}  $\pm 0.00$ & 0.05 $\pm 0.05$ & 0.00 $\pm 0.00$\\

        \bottomrule
    \end{tabular}
\centering
Reported NMI Across Synthetic Low and (embedded) High Dimensional Processes

\begin{figure}[H]
\centering

\begin{subfigure}{0.23\textwidth}
    \centering
    \includegraphics[width=\linewidth]{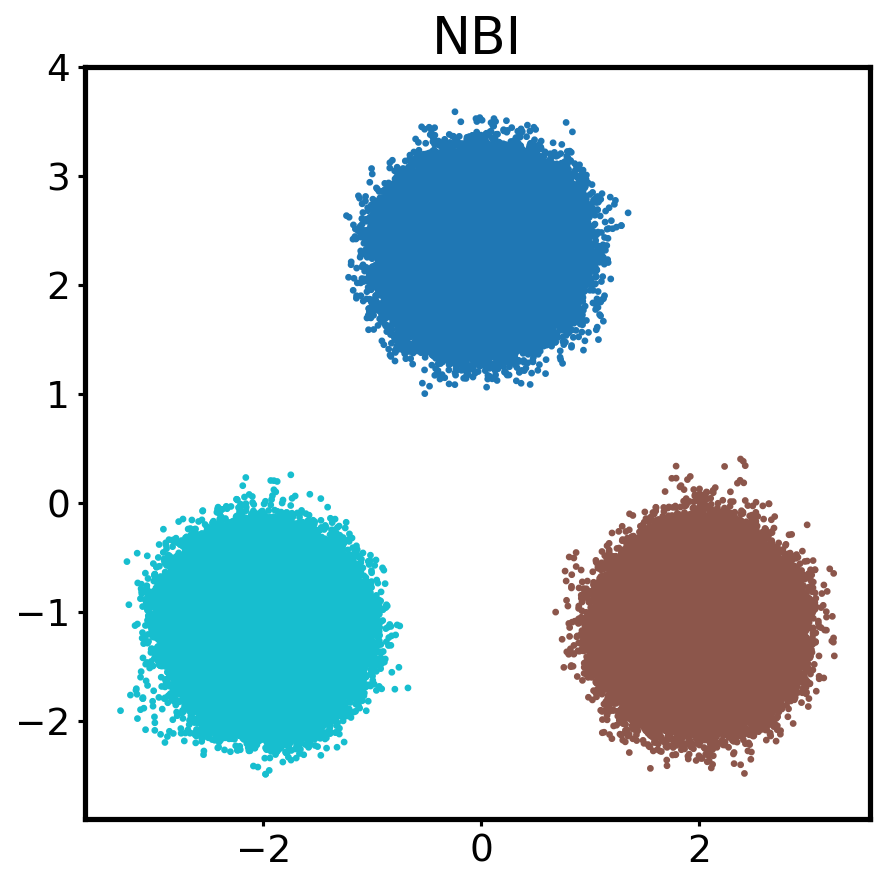}
    \caption{}
\end{subfigure}
\hfill
\begin{subfigure}{0.23\textwidth}
    \centering
    \includegraphics[width=\linewidth]{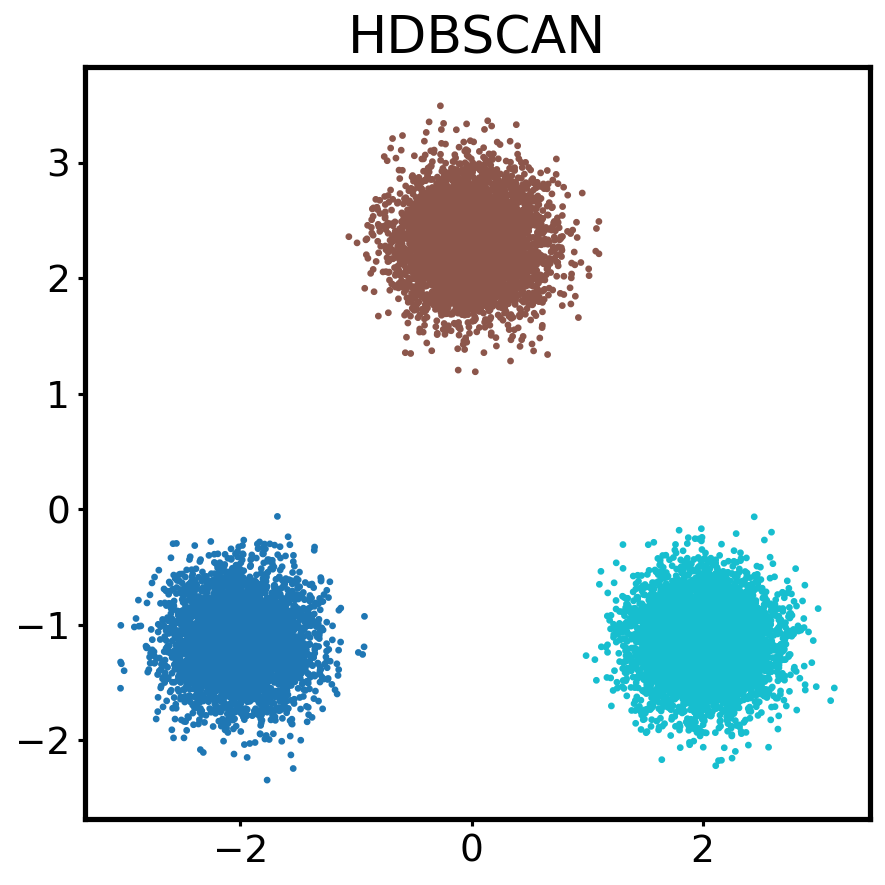}
    \caption{}
\end{subfigure}
\hfill
\begin{subfigure}{0.23\textwidth}
    \centering
    \includegraphics[width=\linewidth]{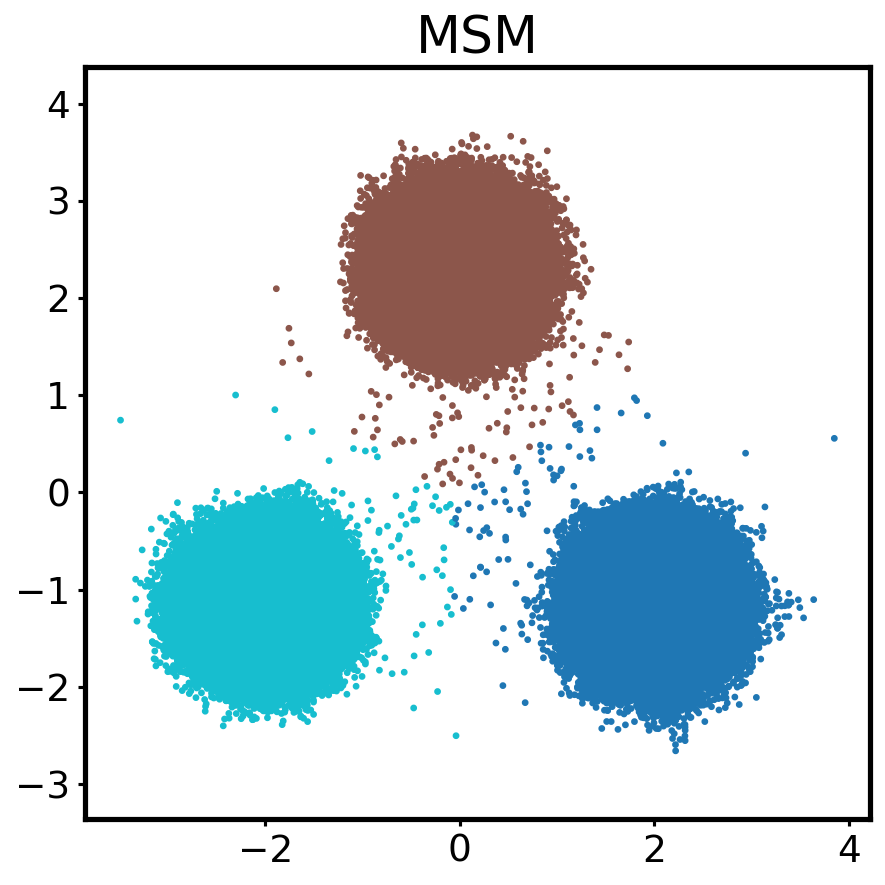}
    \caption{}
\end{subfigure}
\hfill
\begin{subfigure}{0.23\textwidth}
    \centering
    \includegraphics[width=\linewidth]{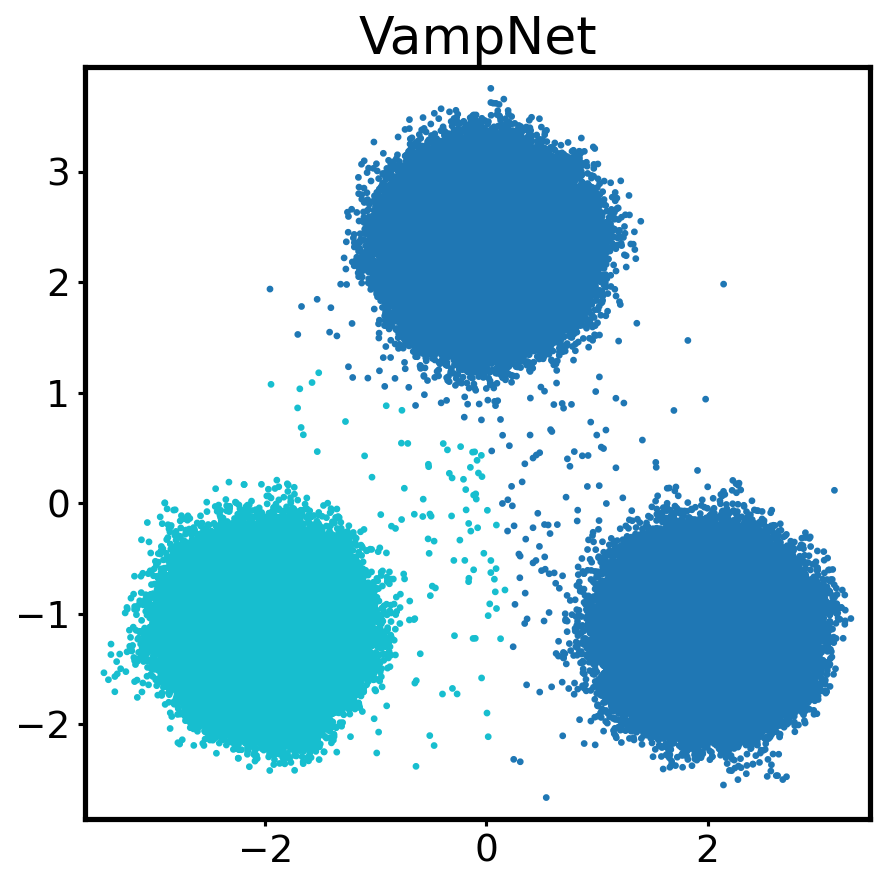}
    \caption{}
\end{subfigure}

\vspace{0.5em}

\begin{subfigure}{0.23\textwidth}
    \centering
    \includegraphics[width=\linewidth]{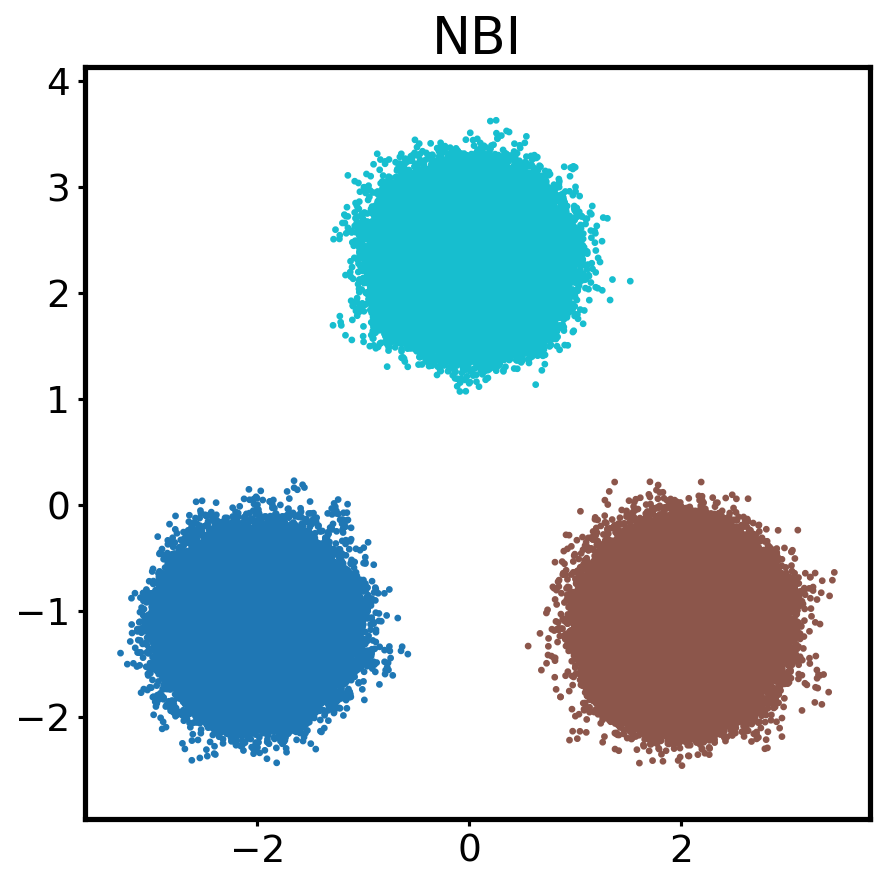}
    \caption{}
\end{subfigure}
\hfill
\begin{subfigure}{0.23\textwidth}
    \centering
    \includegraphics[width=\linewidth]{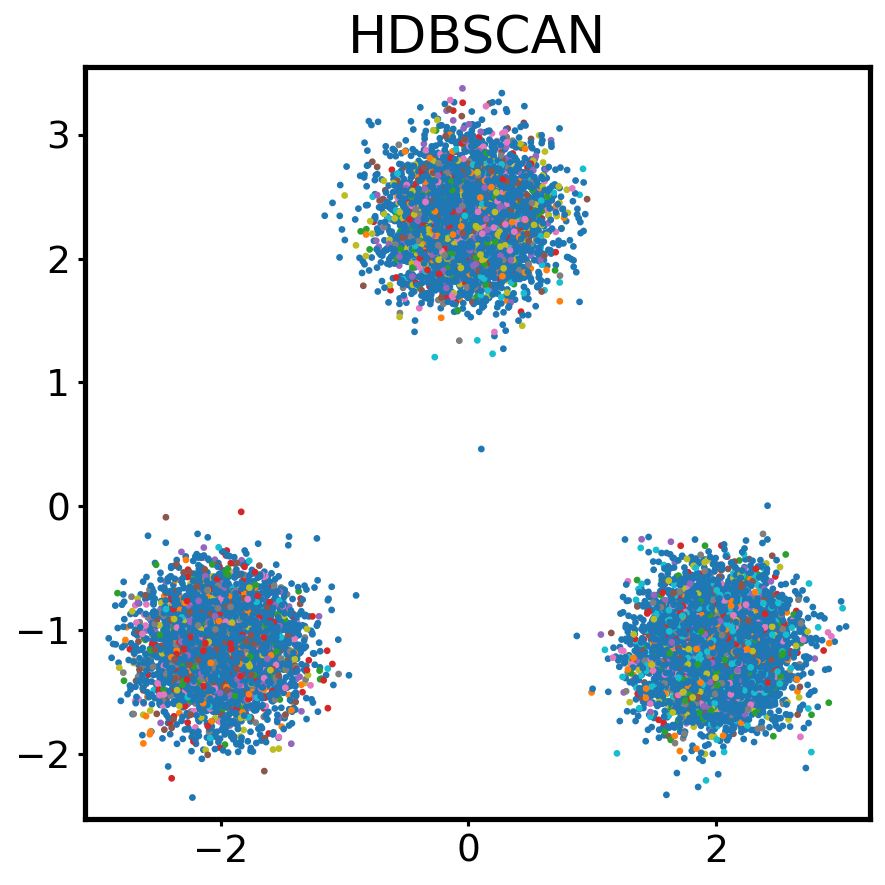}
    \caption{}
\end{subfigure}
\hfill
\begin{subfigure}{0.23\textwidth}
    \centering
    \includegraphics[width=\linewidth]{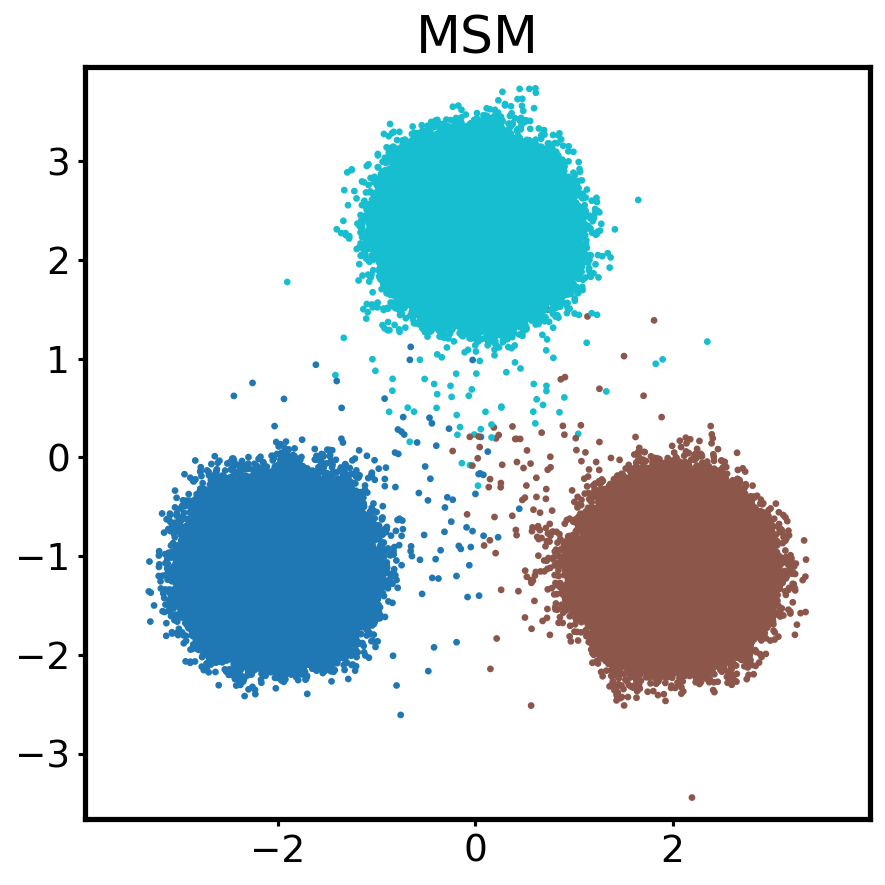}
    \caption{}
\end{subfigure}
\hfill
\begin{subfigure}{0.23\textwidth}
    \centering
    \includegraphics[width=\linewidth]{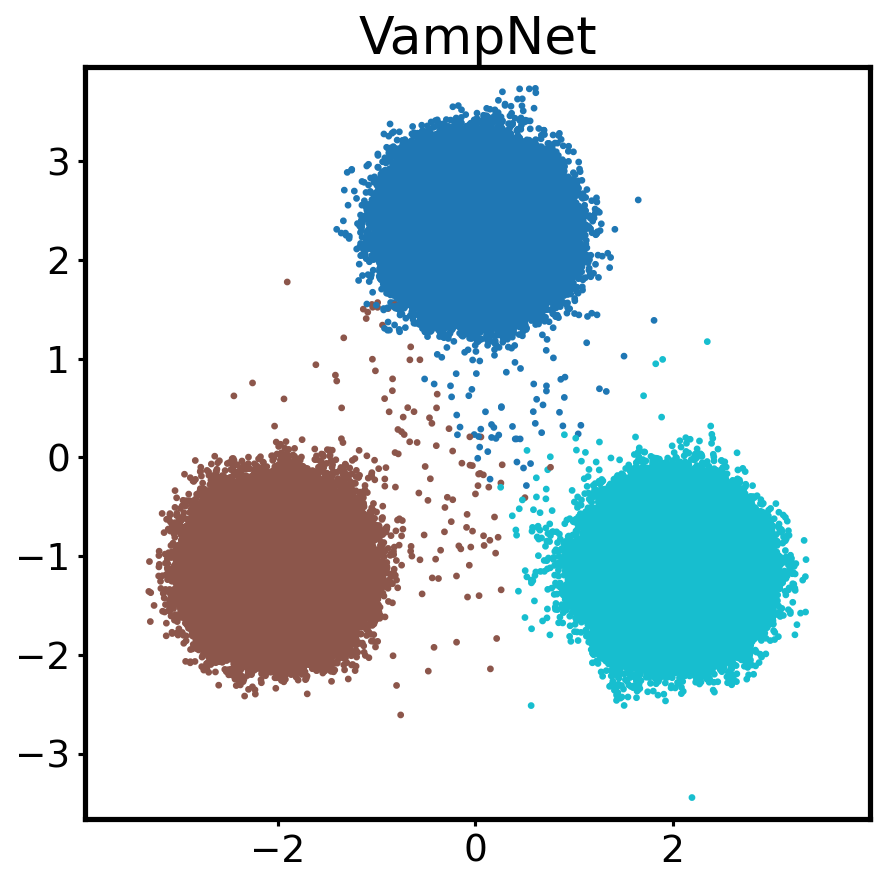}
    \caption{}
\end{subfigure}

\caption{The figure contains trajectories plotted by cluster label for various different methods. The top row corresponds to the MALA process running on a simple 2d Gaussian Mixture energy function composed of two 3 isotropic Gaussians. The bottom row corresponds to the same structure embedded in 100 dimensional ambient space. In this case, we plot only the low dimensional subspace in which the separability of the basins can be visualized. Note that the results of the experiments can have high variance and thus the images presented do not represent the clustering that will be obtained for every experimental run but one randomly selected sample from our experiments. The values reported in our tables are means (along with standard deviation) computed over 10 experimental runs. Some methods may produce a correct clustering on some occasions but not others.} 
\label{gm2d_plot}
\end{figure}

\begin{figure}[H]
\centering

\begin{subfigure}{0.45\textwidth}
    \centering
    \includegraphics[width=\linewidth]{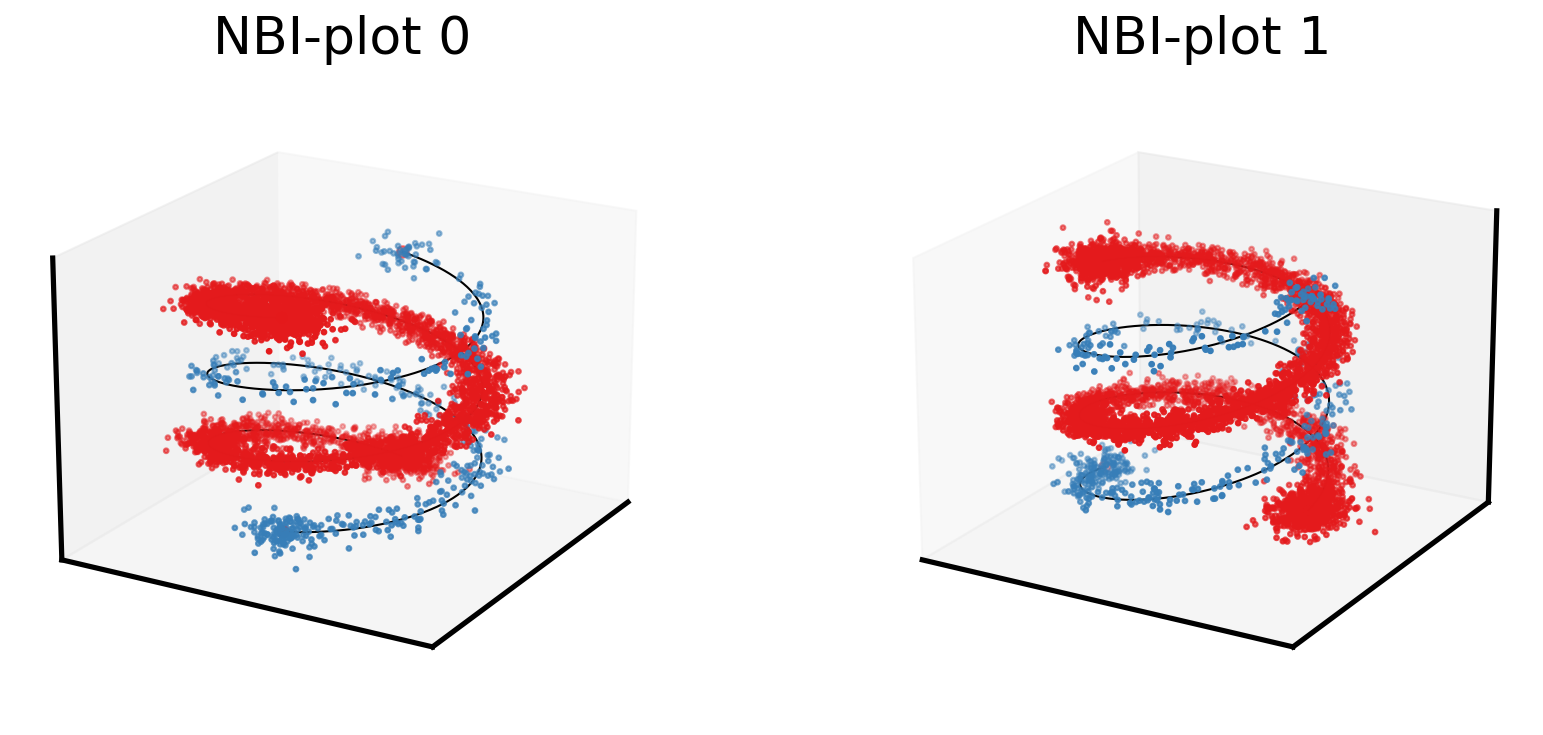}
    \caption{}
\end{subfigure}
\hfill
\begin{subfigure}{0.45\textwidth}
    \centering
    \includegraphics[width=\linewidth]{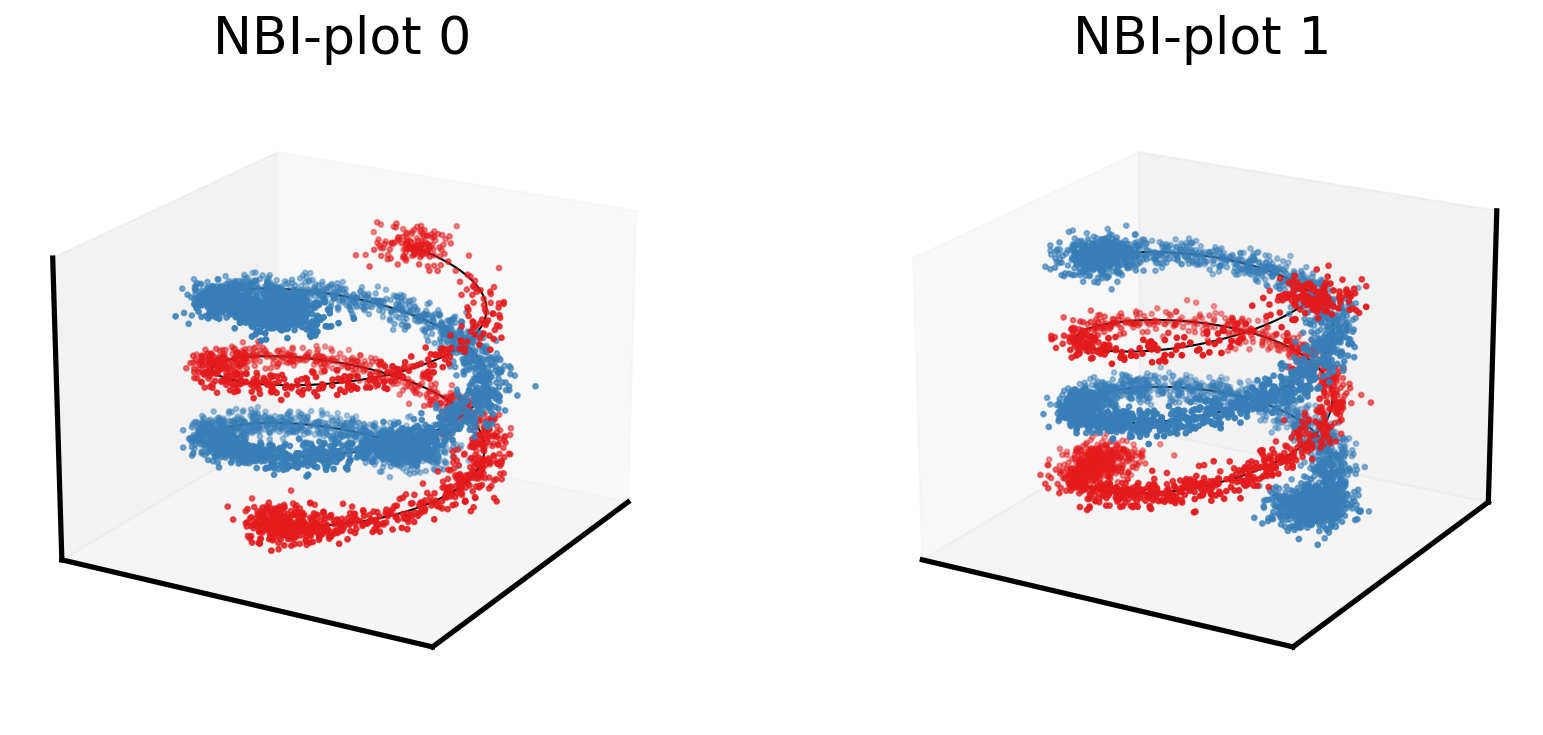}
    \caption{}
\end{subfigure}
\hfill
\begin{subfigure}{0.45\textwidth}
    \centering
    \includegraphics[width=\linewidth]{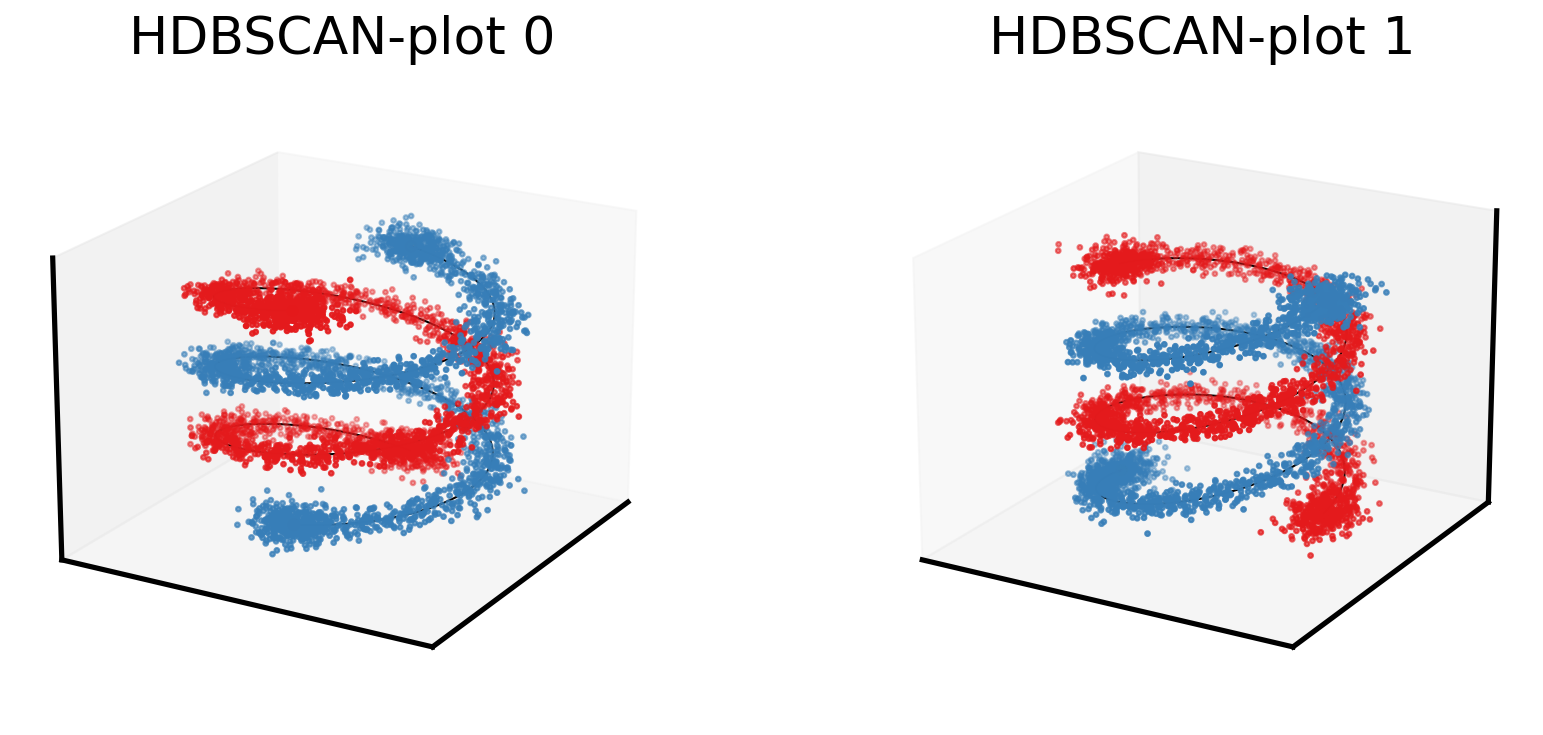}
    \caption{}
\end{subfigure}
\hfill
\begin{subfigure}{0.45\textwidth}
    \centering
    \includegraphics[width=\linewidth]{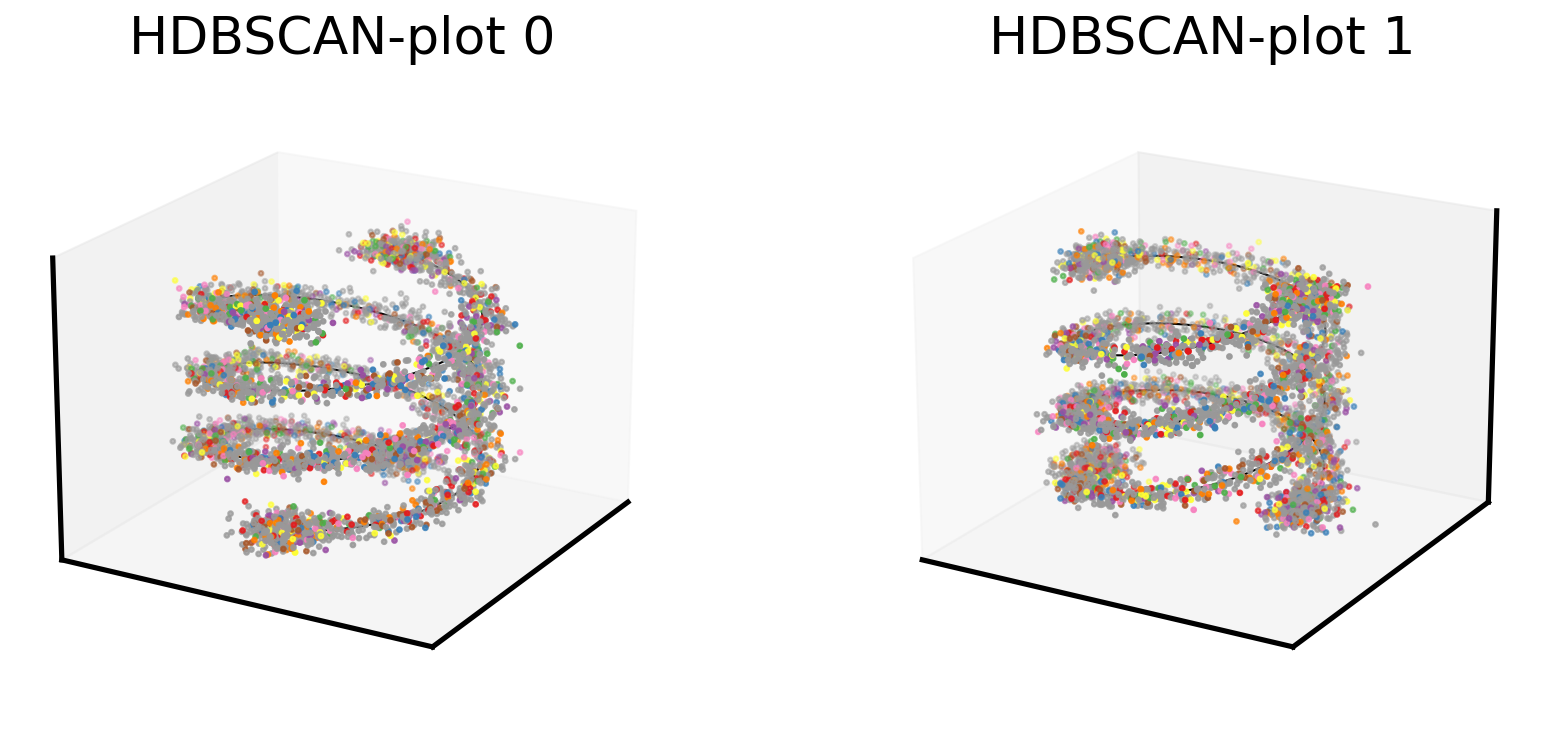}
    \caption{}
\end{subfigure}

\vspace{0.5em}

\begin{subfigure}{0.45\textwidth}
    \centering
    \includegraphics[width=\linewidth]{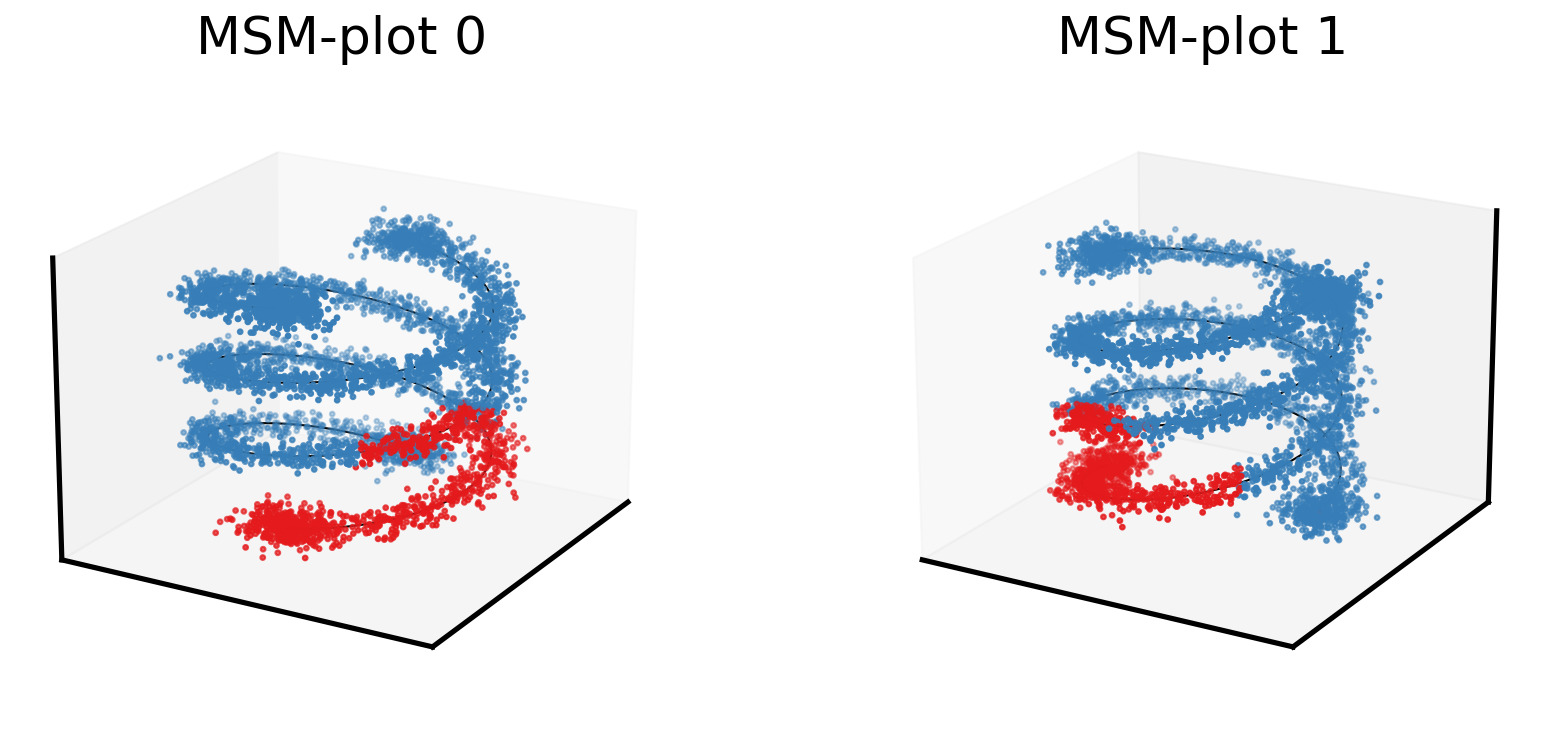}
    \caption{}
\end{subfigure}
\hfill
\begin{subfigure}{0.45\textwidth}
    \centering
    \includegraphics[width=\linewidth]{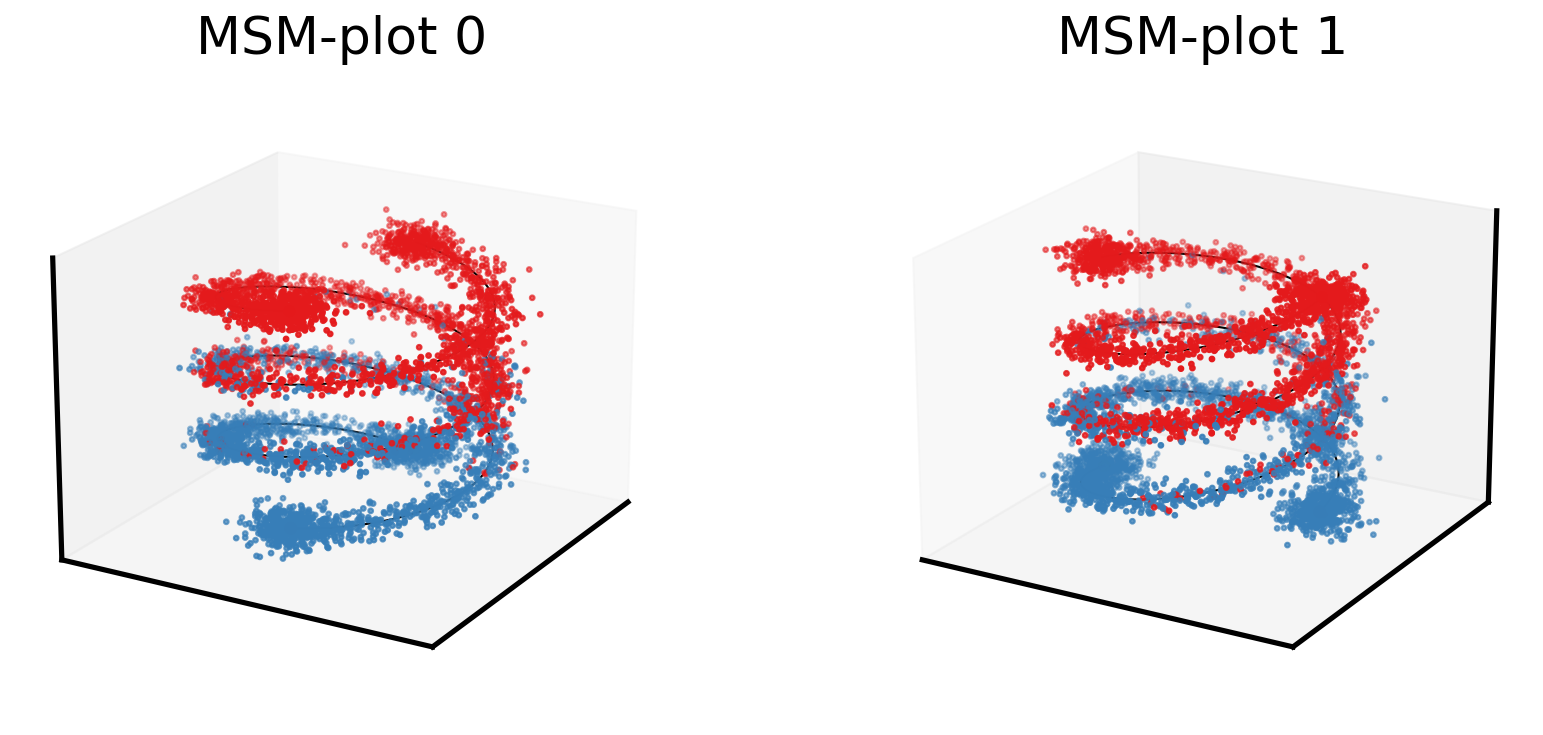}
    \caption{}
\end{subfigure}
\hfill
\begin{subfigure}{0.45\textwidth}
    \centering
    \includegraphics[width=\linewidth]{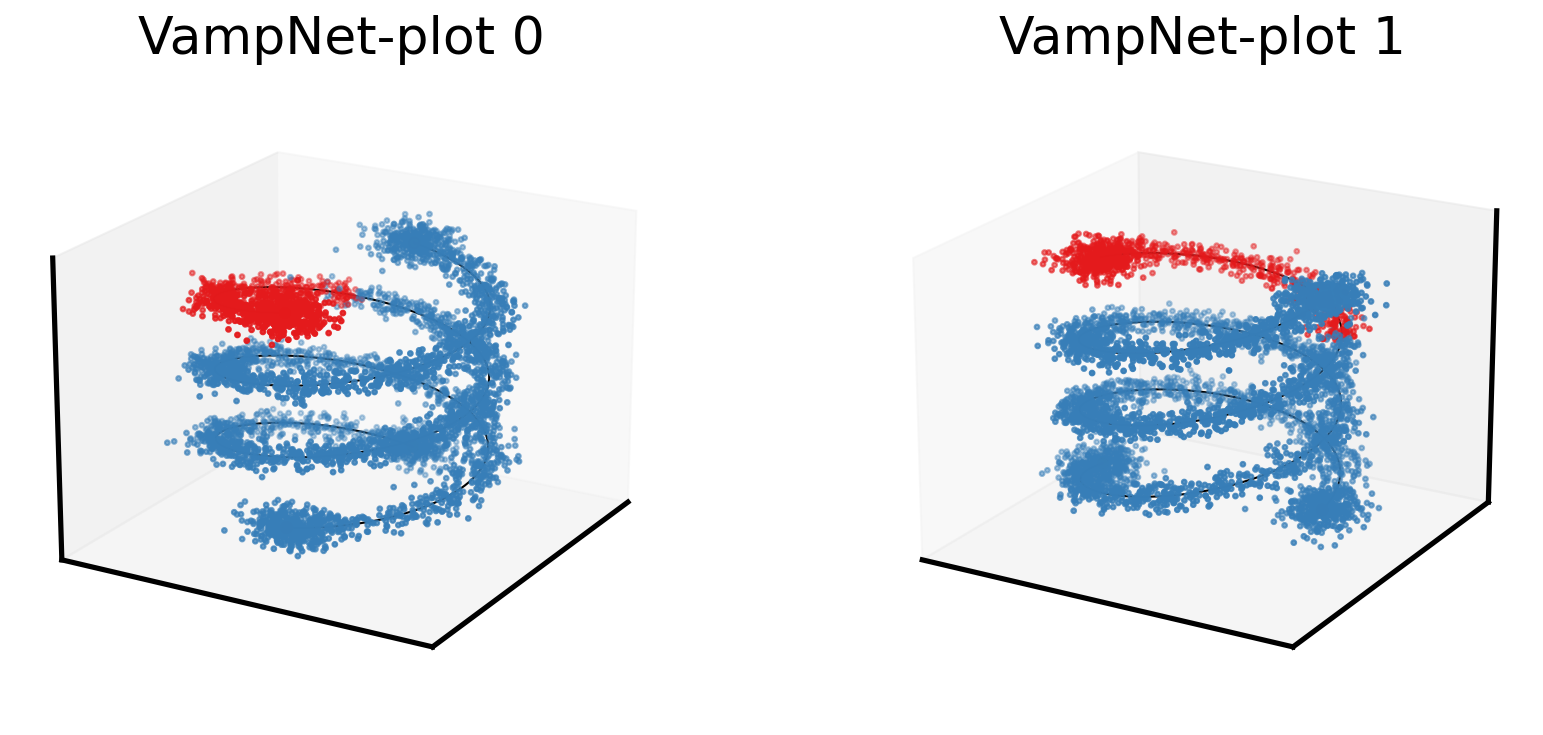}
    \caption{}
\end{subfigure}
\hfill
\begin{subfigure}{0.45\textwidth}
    \centering
    \includegraphics[width=\linewidth]{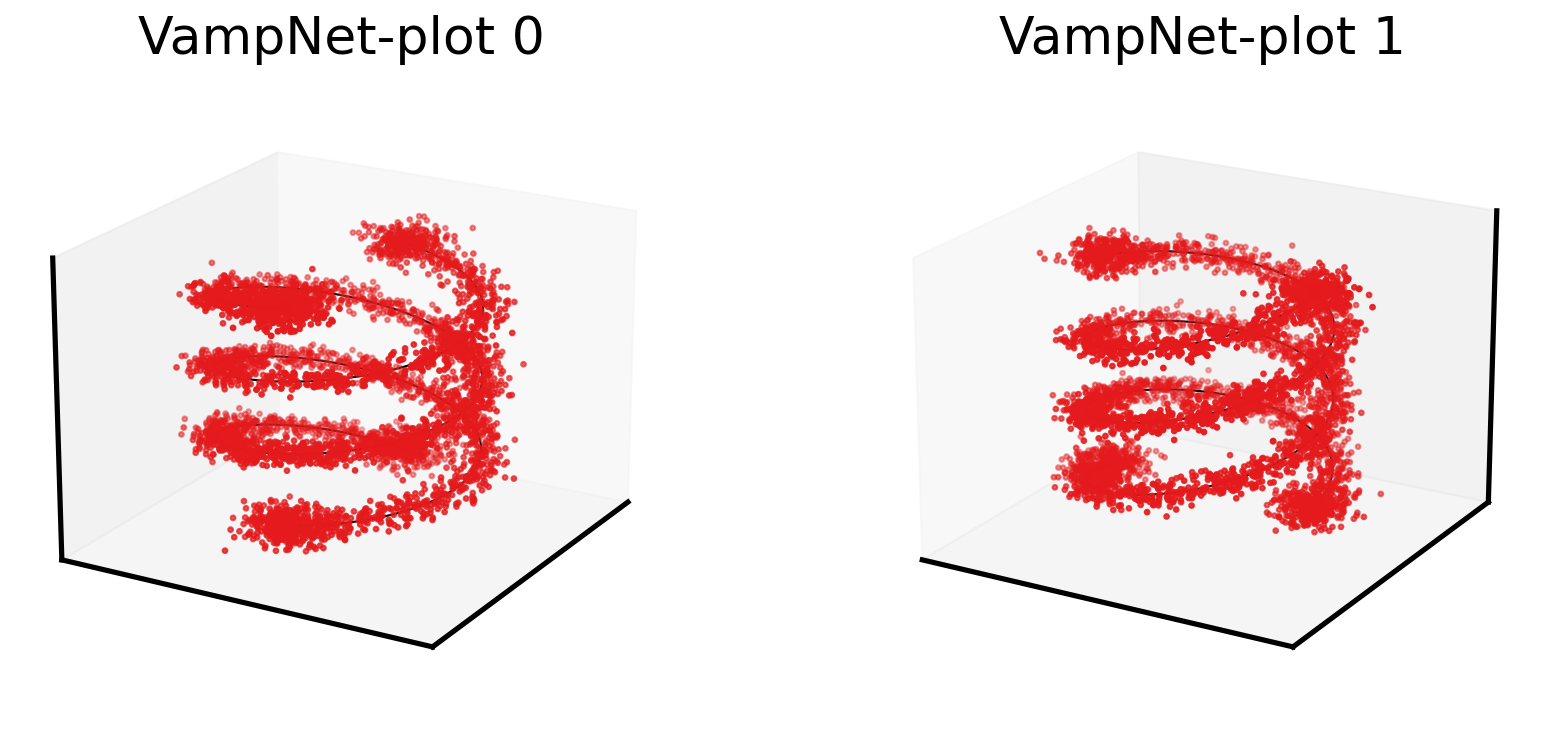}
    \caption{}
\end{subfigure}
\caption{The figure contains trajectories plotted by cluster label for various different methods. The left corresponds to the MALA process running on a 3d 'helix' energy function. The right column corresponds to the same structure embedded in 100 dimensional ambient space. In this case, we plot only the low dimensional subspace in which the separability of the basins can be visualized.}
\label{ring_plot}
\end{figure}

\end{document}